\documentclass[11pt]{article}

\usepackage{graphicx}
\usepackage{hyperref}
\usepackage{multirow}
\usepackage{amssymb,amsmath,amsthm}
\usepackage{latexsym,amsfonts}
\usepackage{color}
\usepackage{graphicx}
  \DeclareGraphicsExtensions{.jpeg,.png,.jpg,.eps,.pdf}
\usepackage{tikz}
\usepackage{epstopdf}

\def\a{{\bf a}}
\def\b{{\bf b}}
\def\c{{\bf c}}

\def\eb{{\bf e}}
\def\f{{\bf f}}
\def\gbold{{\bf g}}

\def\h{{\bf h}}

\def\u{{\bf u}}
\def\v{{\bf v}}
\def\w{{\bf w}}
\def\x{{\bf x}}

\def\z{{\bf z}}

\def\B{{\mathcal B}}
\def\C{{\mathcal C}}

\def\F{{\mathcal F}}

\def\M{{\mathcal M}}

\def\X{{\mathcal X}}

\def\R{{\mathbb R}}

\def\al{\alpha}
\def\d{\delta}
\def\D{\Delta}
\def\e{\epsilon}
\def\g{\gamma}

\def\l{\lambda}

\def\om{\omega}
\def\OM{\Omega}
\def\r{\rho}
\def\s{\sigma}
\def\SI{\Sigma}
\def\t{\tau}
\def\th{\theta}
\def\Th{\Theta}

\def\balpha{{\boldsymbol \alpha}}
\def\bbeta{{\boldsymbol \beta}}

\def\boldeta{{\boldsymbol \eta}}

\def\bth{{\boldsymbol \theta}}
\def\bphi{{\boldsymbol \phi}}

\def\bxi{{\boldsymbol \xi}}
\def\bzeta{{\boldsymbol \zeta}}

\def\zb{\bar{z}}

\def\kai{k \ap \infty}

\def\tai{t \ap \infty}

\def\ap{\rightarrow}

\def\seq{\subseteq}

\def\bp{\{-1,1\}}
\def\bz{{\bf 0}}
\def\imp{\; \Longrightarrow \;}

\def\fa{\; \forall}

\def\as{\mbox{ a.s.}}
\def\wp{\mbox{ w.p.}}

\def\nm{\Vert}

\renewcommand{\and}{\mbox{$\wedge$}}

\def\bxt{\bxi_{t+1}}
\def\bzt{\bzeta_{t+1}}
\def\bths{\bth^*}

\def\Vd{\dot{V}}
\def\gJ{\nabla J}

\newcommand{\bc}{\begin{center}}
\newcommand{\ec}{\end{center}}
\newcommand{\be}{\begin{equation}}
\newcommand{\ee}{\end{equation}}
\newcommand{\bd}{\begin{displaymath}}
\newcommand{\ed}{\end{displaymath}}
\newcommand{\ba}{\begin{array}}
\newcommand{\ea}{\end{array}}
\newcommand{\ben}{\begin{enumerate}}
\newcommand{\een}{\end{enumerate}}
\newcommand{\bit}{\begin{itemize}}
\newcommand{\eit}{\end{itemize}}
\newcommand{\beq}{\begin{eqnarray}}
\newcommand{\eeq}{\end{eqnarray}}
\newcommand{\btab}{\begin{tabular}}
\newcommand{\etab}{\end{tabular}}
\newcommand{\bfig}{\begin{figure}}
\newcommand{\efig}{\end{figure}}
\newcommand{\btp}{\begin{tikzpicture}}
\newcommand{\etp}{\end{tikzpicture}}

\newcommand{\nmm}[1]{ \nm #1 \nm }
\newcommand{\nmeu}[1]{ \nm #1 \nm_2 }
\newcommand{\nmeusq}[1]{ \nm #1 \nm_2^2 }
\newcommand{\nmi}[1]{ \nm #1 \nm_\infty}

\newcommand{\IP}[2]{ \langle #1 , #2 \rangle }

\def\nmsl1{\nm_{{\rm SL1}}}

\definecolor{verm}{rgb}{0.6,0.2,0.2}
\definecolor{purp}{rgb}{0.3,0.1,0.6}
\definecolor{purple}{rgb}{0.4,0.0,0.6}
\definecolor{bggreen}{rgb}{0.1,0.3,0.1}
\definecolor{dgreen}{rgb}{0.1,0.6,0.1}
\definecolor{black}{rgb}{0.0,0.0,0.0}
\definecolor{crim}{rgb}{0.3,0.1,0.1}
\definecolor{dred}{rgb}{0.5,0.1,0.1}

\definecolor{Blue}{cmyk}{0.65,0.13,0,0}
\definecolor{Black}{cmyk}{0,0,0,1}
\definecolor{Red}{cmyk}{0,1,1,0}
\definecolor{Green}{cmyk}{1,0,1,0}
\definecolor{Orange}{cmyk}{0,0.61,0.87,0.1}
\definecolor{Fuchsia}{cmyk}{0.47,0.91,0,0.08}
\definecolor{PineGreen}{cmyk}{0.92,0,0.59,0.25}

\definecolor{verm}{rgb}{0.6,0.2,0.2}
\definecolor{purp}{rgb}{0.3,0.1,0.6}
\definecolor{purple}{rgb}{0.4,0.0,0.6}
\definecolor{bggreen}{rgb}{0.1,0.3,0.1}
\definecolor{dgreen}{rgb}{0.1,0.6,0.1}
\definecolor{black}{rgb}{0.0,0.0,0.0}
\definecolor{crim}{rgb}{0.3,0.1,0.1}
\definecolor{dred}{rgb}{0.5,0.1,0.1}

\newtheorem{corollary}{Corollary}{\bf}{\it}
\newtheorem{definition}{Definition}{\bf}{\it}
\newtheorem{example}{Example}{\bf}{\rm}
\newtheorem{lemma}{Lemma}{\bf}{\it}
\newtheorem{theorem}{Theorem}{\bf}{\it}
{\bf}{\it}
{\bf}{\it}
{\bf}{\rm}

\def\gJ{\nabla J}
\def\gJt{\gJ(\bth_t)}

\def\Vd{\dot{V}}
\def\bD{{\boldsymbol \D}}

\begin{document}

\title{
Convergence Rates for Stochastic Approximation: \\
Biased Noise with Unbounded Variance, and Applications
}

\author{Rajeeva L.\ Karandikar and M.\ Vidyasagar
\thanks{RLK is Professor Emeritus at the Chennai Mathematical Institute;
Email: rlk@cmi.ac.in.
MV is National Science Chair at Indian Institute of Technology Hyderabad;
Email: m.vidyasagar@iith.ac.in.
The research of MV was supported by the Science and Engineering Research
Board, India.}
}

\maketitle

\begin{abstract}
\textbf{This paper is dedicated to the memory of Boris Teodorovich Polyak.}

In this paper, we study the convergence properties of the Stochastic
Gradient Descent (SGD) method for finding a stationary point
of a given objective function $J(\cdot)$.
The objective function is not required to be convex.
Rather, our results apply to a class of ``invex'' functions, which have the
property that every stationary point is also a global minimizer.
First, it is assumed that $J(\cdot)$ satisfies a property that
is slightly weaker than the Kurdyka-Lojasiewicz (KL) condition,
denoted here as (KL').
It is shown that the iterations $J(\bth_t)$ converge almost surely
to the global minimum of $J(\cdot)$.
Next, the hypothesis on $J(\cdot)$ is strengthened from (KL') to
the Polyak-Lojasiewicz (PL) condition.
With this stronger hypothesis, we derive estimates on the rate of
convergence of $J(\bth_t)$ to its limit.
Using these results, we show that for functions satisfying the PL property,
the convergence rate of both the objective function
and the norm of the gradient with SGD is the same as the best-possible rate for convex
functions.
While some results along these lines have been published in the past,
our contributions contain two distinct improvements.
First, the assumptions on the stochastic gradient are more general
than elsewhere, and second, our convergence is almost sure, and not
in expectation.
We also study SGD when only function evaluations are permitted.
In this setting, we determine the ``optimal'' increments or the size
of the perturbations.
Using the same set of ideas, we establish the global convergence
of the Stochastic Approximation (SA) algorithm under more general
assumptions on the measurement error, compared to the existing literature.
We also derive bounds on the rate of convergence of the SA algorithm
under appropriate assumptions.

\end{abstract}

\section{Introduction}\label{sec:Intro}

Suppose $\f : \R^d \ap \R^d$ is some function, and it is desired to
find a solution $\bths$ to the equation $\f(\bths) = \bz$.
The \textbf{stochastic approximation (SA)} algorithm, introduced in
\cite{Robbins-Monro51}, addresses the situation where the only information
available is a \textit{noise-corrupted} measurement of $\f(\bth_t)$.
If $\gbold : \R^d \ap \R^d$ and it is desired to find a fixed
point of this map, then this is the same as solving $\f(\bths) = \bz$,
where $\f(\bth) := \gbold (\bth) - \bth$.
On the other hand, if $J : \R^d \ap \R$ is a $\C^1$ function and it is desired
to find a stationary point of $J(\cdot)$, then the problem is to find a
solution to $\gJ(\bths) = \bz$, using only a
\textit{stochastic gradient}.

Suppose the problem is one of finding a solution to $\f(\bths) = \bz$.
At step $t$, the avaiable measurement is of
the form $\f(\bth_t) + \bxi_{t+1}$, where $\bxi_{t+1}$ is the error term.
The canonical step in the SA algorithm is to update $\bth_t$ 
to $\bth_{t+1}$ via
\be\label{eq:111}
\bth_{t+1} = \bth_t + \al_t [ \f(\bth_t) + \bxt ] ,
\ee
where $\al_t \in (0,\infty)$ is called the ``step size.''
If one wishes to find a fixed point of the map $\gbold(\cdot)$,
then by defining $\f(\bth) = \gbold(\bth) - \bth$,
one can apply the iteration \eqref{eq:111}, which now takes the form
\be\label{eq:112}
\bth_{t+1} = (1 - \al_t) \bth_t + \al_t [ \gbold(\bth_t) + \bxt ] .
\ee
It is clear that \eqref{eq:112} need not be analyzed separately from
\eqref{eq:111}.
Next, if it is desired to find a stationary point of a $\C^1$-map
$J(\cdot)$, then the iteration becomes
\be\label{eq:112a}
\bth_{t+1} = \bth_t - \al_t \h_{t+1}  ,
\ee
where $\h_{t+1}$ is a noisy approximation to $\gJ(\bth_t)$,
known as the ``stochastic gradient.''
Since the update of $\bth_t$ is in the negative direction of $\h_{t+1}$,
\eqref{eq:112a} is called the \textbf{Stochastic
Gradient Descent (SGD)} method.
Note that in this paper we study only optimization problems
where the variable $\bth$ is \textit{unconstrained}.
If $\bth$ is restricted to belong to some closed convex subset
$S \seq \R^d$, then $\h_{t+1}$ would be an approximate \textit{projection}
of the gradient.
Methods such as mirror descent incorporate such a projection.
However, since we study only unconstrained problems, one can think of
$\h_{t+1}$ as an approrimate gradient.

In this paper, we establish the convergence of the SA algorithm of
\eqref{eq:111} and the SGD algorithm of \eqref{eq:112a}
under more general (i.e., less restrictive) assumptions than at present.
We establish not just convergence, but also bounds on the \textit{rates}
of convergence.
Specifically for SGD, we establish \textit{almost sure} convergence,
as opposed to convergence in expectation as in much of the literature.
Since any stochastic algorithm results in a single sample path of
a stochastic process, it is very useful to know that almost all sample
paths converge to the desired limit.
We establish almost sure convergence for SA as well; however, that is
the usual practice in that literature.

The paper is organized as follows: 
In Section \ref{sec:Contrib}, we briefly summarize the contributions
of the present paper, above and beyond the known results.
In Section \ref{sec:Classical}, we review the early classical results
in Stochastic Approximation and Stochastic Gradient Descent, roughly before
the year 2000.
In Section \ref{sec:Recent}, we review results in more recent times,
with emphasis on the variety of definitions of a ``stochastic gradient,''
and various properties that the objective function is assumed to satisfy.
In Section \ref{sec:New}, we present two general theorems on the
convergence of stochastic processes.
While these theorems form the basis for the proofs in later sections,
they might be of independent interest.
In Section \ref{sec:SGD}, we apply the convergence theorems of
Section \ref{sec:New} to establish the convergence of the SGD algorithm,
and also to obtain bounds on the rate of convergence.
In Section \ref{sec:SA}, these same
theorems are applied to study the SA algorithm, and results analogous
to those in Section \ref{sec:SGD} are proved.
Section \ref{sec:Conc} suggests a few problems for future research.

\section{Contributions of the Paper}\label{sec:Contrib}

The focus of the paper is on establishing the almost sure convergence
of the Stochastic Gradient Descent (SGD) algorithm 
under the most general conditions thus far.
Specifically, we have achieved the following:
\ben
\item The class of functions for which the convergence of SGD
is established includes not only convex functions, but also some
nonconvex functions.
All the functions studied here are ``invex'' in the sense that every
stationary point is also a global minimum.
(But there are invex functions that are not covered by our approach.)
When the objective function $J(\cdot)$ satisfies an analog of the
Kurdyka-Lojasiewicz property (our condition is slightly weaker),
we can prove the almost sure convergence of SGD.
If $J(\cdot)$ satisfies the stronger Polyak-Lojasiewicz property,
we not only establish the convergence of SGD, but also bounds on the
\textit{rate} of convergence.
\item
Previously, estimates were, for the most part, available only for
convergence in expectation.
We are able to estimate the rate of \textit{almost sure} convergence as well.
For this, we build on the contents of 
\cite{Sebbouh-Gower-Defazio-PMLR21,Liu-Yuan-arxiv22}.
However, unlike in those papers, our estimates do not require a specific
choice of step size sequences, but are quite general.
This point is elaborated further in Section \ref{sec:New}.
See the Remark after Theorem \ref{thm:52}.
\item 
The assumptions on the stochastic gradient are the most general thus far.
Specifically,
define the quantities $\x_t$ and $\bzt$ as in \eqref{eq:1221a} below.
Roughly speaking, $\x_t$ is the ``bias'' of the stochastic gradient
$\h_{t+1}$, that is, the difference between the conditional expectation
of $\h_{t+1}$ and the true gradient $\gJt$; and $\bzt$ is the
``unpredictable part'' of $\h_{t+1}$.
Then our assumptions are
\bd
\nmeu{\x_t} \leq \mu_t [ 1 + \nmeu{\gJt} ] , \fa \bth_t \in \R^d , \fa t ,
\ed
\bd
CV_t(\h_{t+1}) = E_t ( \nmeusq{\bzt} ) \leq M_t^2 [ 1 + J(\bth_t) ] ,
\fa \bth_t \in \R^d , \fa t .
\ed
While bounds on $CV_t(\h_{t+1})$ similar to the above are found in,
for example, \cite[Section 4]{Bottou-et-al-SIAM18},
the presence of the term $\nmeu{\gJt}$ in the first equation is new.
As shown in Section \ref{sec:SGD}, our results apply (for example)
to coordinate gradient descent, when the component to be updated is
\textit{not} selected according to a uniform distribution across components.
Further, our assumption on the conditional variance of the stochastic
gradient is weaker than the so-called ``expected smoothness'' condition
from \cite{Khaled-Rich-arxiv20}, which is proposed as the ``weakest 
assumption.''
We show that if $\nabla J(\cdot)$ is globally Lipschitz-continuous,
which is a standard assumption, then there is no need for the expected
smoothness assumption.
\item It is shown that when $J(\cdot)$ satisfies the
Polyak-Lojasiewicz property, the rate of converges matches
the ``best possible'' rate for SGD established in \cite{Arjevani-et-al-MP23}
for convex functions, in terms of both the
objective function and the norm of the gradient.
This result is significant because it is also shown in 
\cite{Arjevani-et-al-MP23} that, for arbitrary nonconvex functions, the 
achievable convergence rate of GSD is much slower.
Thus, effectively, we are able to extend the rates proved in
\cite{Arjevani-et-al-MP23} to a class of nonconvex functions.
\item Next we study stochastic gradients that use function evaluations alone.
We establish the ``optimal'' choice of increments for achieving the fastest
convergence.
Using this optimal choice, it is shown that by using \textit{three}
function evaluations per iteration, it is possible to match the convergence
rate in \cite{Nesterov-FCM17}, though this paper is restricted to convex
objective functions and noise-free measurements.
\item Finally, the same methods used to establish the convergence of SGD
are also used to establish the convergence of Stochastic Approximation.
Specifically, under very general assumptions similar to those in SGD,
we build upon the results in \cite{MV-MCSS23}. 
We not only
prove almost sure convergence, but also bound the rates of convergence.
Because of the similarity of the proofs, these theorems are stated
and the proofs are just sketched.
\een

\section{Classical Results in SA and SGD}\label{sec:Classical}

In order to place our contributions to SA in context,
we begin with the classical results.
For the benefit of the reader, we state all results using the notation
of the present paper.
Recent results are described in Section \ref{sec:Recent}.

Different analyses of the SA and SGD algorithms depend on different
assumptions on the error $\bxt$ in \eqref{eq:111},
and the nature of the stochastic gradient $\h_{t+1}$ in \eqref{eq:112a}.
In order to describe the classical and recent results concisely,
we introduce some notation.
Throughout the paper, all random variables and stochastic processes
are defined on some underlying probability space $(\OM,\SI,P)$
where $\OM$ is the sample space, $\SI$ is a $\s$-algebra denoting the
event space, and $P$ is a probability measure defined on $\SI$.
Further, throughout the paper,
all random variables are assumed to be square-integrable, so that
various conditional expectations and variances are well-defined.
Let $\bth_0^t$ denote $(\bth_0 , \cdots , \bth_t)$, and similarly
let $\bxi_1^t$ denote $\bxi_1 , \cdots , \bxi_t$, 
let $\h_1^t$ denote $(\h_1 , \cdots , \h_t)$.
Note that there is no $\bxi_0$ nor an $\h_0$.
The initial guess $\bth_0$ in SA can be either deterministic or random.
Let $\F_t$ denote the $\s$-algebra generated by $\bth_0 , \bxi_1^t$
in the case of \eqref{eq:111} or \eqref{eq:112}, and the $\s$-algebra
generated by  $\bth_0 , \h_t^t$ in the case of \eqref{eq:112a}.
Let $\M(\F_t)$ denote the set of functions that are measurable
with respect to $\F_t$.
Then it is clear that $\bth_t \in \M(\F_t)$ for all $t \geq 1$.
For an $\R^d$-valued
 random variable $X$, let $E_t(X)$ denote the \textbf{conditional
expectation} $E(X | \F_t)$, and let $CV_t(X)$ denote its 
\textbf{conditional variance} defined by\footnote{See 
\cite{Williams91,Durrett19} for relevant background on stochastic processes.}
\be\label{eq:112b}
CV_t(X) = E_t( \nmeusq{ X - E_t(X)} ) = E_t(\nmeusq{X}) - \nmeusq{E_t(X)} .
\ee

The SA algorithm was introduced in \cite{Robbins-Monro51}
for the scalar case where $d = 1$.
However, we state it for the multidimensional case, and in our notaton.
In this formulation,
the error $\bxt$ is assumed to satisfy the following assumptions:
\be\label{eq:112c}
E_t ( \bxt ) = \bz , \quad CV_t( \bxt ) \leq M^2 , \as ,
\ee
for some finite constant $M$.\footnote{
Note that, since the paper deals with random variables
and stochastic processes, \textit{almost all statements} hold
``almost surely.''
To avoid tedious repetition, we omit this phrase in what follows.}
The first assumption implies that $\{ \bxt \}$ is a martingale
difference sequence, and also that
$\f(\bth_t) + \bxt$ is an \textit{unbiased} measurement of $\f(\bth_t)$. 
The second assumption
means that the conditional variance of the error is globally bounded,
both as a function of $\bth_t$ and as a function of $t$.
With the assumptions in \eqref{eq:112d} below, along with some assumptions
on the function $\f(\cdot)$, it is shown in \cite{Robbins-Monro51}
that $\bth_t$ converges to a solution of $\f(\bths) = \bz$,
provided the step size sequence $\{ \al_t \}$
satisfies the \textbf{Robbins-Monro (RM)} conditions
\be\label{eq:112d}
\sum_{t=0}^\infty \al_t^2 < \infty , \quad
\sum_{t=0}^\infty \al_t = \infty .
\ee

The first SGD method was introduced in \cite{Kief-Wolf-AOMS52},
for finding a stationary point of a $\C^1$
function $J: \R \ap \R$, that is, a solution to
$\gJ(\bth) = \bz$.\footnote{
Strictly speaking, we should use $J'(\th)$ for the scalar case.
But we use vector notation to facilitate comparison with later formulas.}
using an \textit{approximate gradient} of $J(\cdot)$.
The specific formulation used in \cite{Kief-Wolf-AOMS52} is
\be\label{eq:112e}
h_{t+1} := \frac{ J(\th_t + c_t \D + \xi_{t+1}^+) -
J(\th_t - c_t \D + \xi_{t+1}^-)}{2 c_t} 
\approx \gJ(\th_t) .
\ee
where $\D$ is a small and fixed real number,
$c_t >0$ is called the \textbf{increment},
and $\xi_{t+1}^+$, $\xi_{t+1}^-$ are the measurement errors.
This terminology ``increment'' is not standard but is used here.
In order to make the expression a better and better approximation to
the true $\gJ(\bth_t)$, the increment $c_t$ must approach zero as $\tai$.
This approach was extended to the multidimensional case in \cite{Blum54}.
There are several ways to extend \eqref{eq:112e} to the multivariate case,
and these are discussed in Section \ref{sec:Recent}.
Let $h_{t+1}$ denote the (scalar) stochastic gradient defined in
\eqref{eq:112e}, and define
\bd
z_t = E_t ( h_{t+1} ) , x_t = z_t - \gJ(\th_t) , \zeta_{t+1} = h_{t+1} - z_t .
\ed
Then it is shown in \cite{Kief-Wolf-AOMS52} that the error term satisfies
\be\label{eq:112f}
| E_t ( \zeta_{t+1}) | \leq K c_t , \quad CV_t(\zeta_{t+1}) \leq M^2/c_t^2 ,
\ee
for suitable constants $K, M$.
In other words, neither of the two assumptions in \eqref{eq:112c} is satisfied:
The estimate of $\gJ(\bth_t)$ is biased, and the variance is unbounded
as a function of $t$, though it is bounded as a function of $\bth_t$
for each fixed $t$.
In this case, for the scalar case, it was shown in \cite{Kief-Wolf-AOMS52}
that $\bth_t$ converges to a stationary point of $J(\cdot)$
if the Kiefer-Wolfwitz-Blum (KWB) conditions
\be\label{eq:112g}
c_t \ap 0 ,  \quad 
\sum_{t=0}^\infty ( \al_t^2 / c_t^2 ) < \infty , \quad 
\sum_{t=0}^\infty \al_t c_t < \infty , \quad 
\sum_{t=0}^\infty \al_t = \infty 
\ee
are satisfied.
In \cite{Blum54} it is shown that the same conditions also ensure
convergence when $d > 1$.
Note that the conditions automatically imply the finiteness of
the sum of $\al_t^2$.

One of the first papers to expand the scope of SA is
\cite{Kushner-JMAA77}.
In that paper, the author considers a recursion of the form
\bd
\bth_{t+1} = \bth_t - \al_t \gJ(\bth_t) + \al_t \bxt + \al_t \bbeta_{t+1} ,
\ed
where $\bbeta_t \ap \bz$ as $\tai$.
Here, the sequence $\{ \bxt \}$ is \textit{not}
assumed to be a martingale difference sequence.
Rather, it is assumed to satisfy a different set of conditions,
referred to as the Kushner-Clark conditions;
see \cite[A5]{Kushner-JMAA77}.
It is then shown that if the error sequence $\{ \bxt \}$ satisfies
\eqref{eq:112c}, i.e., is a martingale difference sequence,
then Assumption (A5) holds.
Essentially the same formulation is studied in \cite{Ljung78}.
The same formulation is also studied
\cite[Section 2.2]{Borkar22}, where \eqref{eq:112c} holds,
and $\bbeta_t \ap \bz$ as $\tai$.
In \cite{Tadic-Doucet-AAP17}, it is assumed only that
$\limsup_t \bbeta_t < \infty$.

In all cases,
it is shown that $\bth_t$ converges to a solution of $\f(\bths) = \bz$,
\textit{provided} the iterations remain bounded almost surely.
However, this is a very strong assumption, in our view.
The assumption that $\beta_t \ap 0$ as $\tai$ may not, by itself,
be sufficient to ensure that
the iterations are bounded, as shown by the next simple example.
Consider the \textit{deterministic scalar} recursion
\bd
\th_{t+1} = (1 + \al_t) \th_t + \al_t \beta_t ,
\ed
where $\{ \beta_t \}$ is a sequence of constants.
The closed-form solution to the above recursion is
\bd
\th_t = \prod_{\t=0}^{t-1} (1+ \al_\t) \th_0
+ \sum_{k=0}^{t-1} \left[ \prod_{s=k}^{t-1} (1 + \al_s) \right] \al_k \beta_k.
\ed
Now let $\th_0 = 0$ and suppose $\beta_t \geq 0$ for all $t$.
Then it follows that
\bd
\th_t = \sum_{k=0}^{t-1} \left[ \prod_{s=k}^{t-1} (1 + \al_s) \right]
\al_k \beta_k
\geq  \sum_{k=0}^{t-1} \al_k \beta_k .
\ed
Thus, even when the step size sequence $\{ \al_t \}$ satisfies the
standard Robbins-Monro conditions,
it is possible to choose the sequence $\{ \beta_t \}$ in such a manner
that $\beta_t \ap 0$ as $\tai$, and yet
\bd
\sum_{t=0}^\infty \al_t \beta_t = \infty .
\ed
Thus, merely requiring that $\beta_t \ap 0$ as $\tai$ is not sufficient
to ensure the boundedness of the iterations.
This discussion shows that there is a need for an approach in which
the boundedness of the iterations can be inferred separately from
the convergence.

%

In most of the references mentioned thus far,
the convergence of the SA algorithm is proved using the so-called ODE method.
This approach is based on the idea that, as $\al_t \ap 0$, the
sample paths of the stochastic process ``converge'' to the
\textit{deterministic} solutions of the associated ODE
$\dot{\bth} = \f(\bth)$.
This approach is introduced in \cite{Kushner-JMAA77,Ljung-TAC77b,Der-Fradkov74}.
Book-length treatments of this approach can be found in
\cite{Kushner-Clark12,Kushner-Yin03,BMP90,Borkar22}.
See also \cite{Meti-Priou84} for an excellent summary.
In principle, the ODE method can cope with the situation where
the equation $\f(\bths) = \bz$ has multiple solutions.
The typical theorem in this approach states that \textit{if}
the iterations $\{ \bth_t \}$ remain bounded, then $\bth_t$
approaches the solution set of the equation under study.
In \cite{Borkar-Meyn00}, for the first time, the boundedness
of the iterations is a \textit{conclusion, not a hypothesis}.
The arguments in that paper, and its successors, are based on defining
a \textbf{mean flow equation}
\bd
\dot{\bth} = \f_\infty(\bth) , \quad 
\f_\infty(\bth) := \lim_{r \ap \infty} \frac{\f(r \bth)}{r} .
\ed
It is assumed that $\f(\cdot)$ is globally Lipschitz-continuous and
that $\bz$ is a globally asymptotically stable equilibrium of
$\f_\infty(\cdot)$.
This implies that the equation under study has a unique solution,
in effect negating one of the potential advantages of the ODE method.
Also, it is easy to see that if $\f(\cdot)$ grows \textit{sublinearly}, i.e.,
\bd
\lim_{\nmm{\bth} \ap \infty} \frac{\nmm{\f(\bth)}}{\nmm{\bth}} = 0 ,
\ed
then $\f_\infty (\bth) \equiv \bz$, so that the hypothesis of
\cite{Borkar-Meyn00} can never be satisfied.
In addition, when $\f(\cdot)$ is discontinuous, the limiting equation
is not an ODE, but a differential inclusion; see for example
\cite{Borkar-Shah-arxiv23}; this requires more subtle analysis.

In contrast, the analysis in this paper is based on the so-called
``supermartingale approach,'' pioneered in \cite{Gladyshev65,Robb-Sieg71}.
In contrast with the ODE approach, the supermartingale approach can cope with 
functions $\f(\cdot)$ that grow sublinearly and/or are discontinuous,
with no modifications.
In this approach, it is also very easy to obtain bounds on the
\textit{rate of convergence} of the algorithm.
The presumed advantage of the ODE method is that it can cope with 
the case where there are multiple solutions; this comes at the expense
of \textit{assuming} rather than \textit{inferring} that the iterations
remain bounded.
In the supermartingale approach, not only is it easy to infer the
boundedness of the iterations, but boundedness can be inferred separately
from convergence.
Finally, the analysis remains virtually unchanged when the step sizes
$\al_t$ are themselves \textit{random.}
Random step sizes are natural when ``block'' updating is used in
\eqref{eq:111} or \eqref{eq:112a}; see Section \ref{ssec:32} for a mention
of block updating.

\section{Review of SGD for Nonconvex Optimization}\label{sec:Recent}

In this paper we aim to study the minimization of a class of \textit{nonconvex}
$\C^1$ objective functions, which have the property that every stationary
point is also a global minimum.
While every smooth ($\C^1$) convex function has this property, so do
some nonconvex functions, for example ``invex'' functions (see below).
In this section, we briefly survey the recent literature in the area of 
the Stochastic Gradient Method (SGD) applied to nonconvex optimization.
Given that the literature is vast even within these limits, we refer the
reader to the survey paper \cite{Bottou-et-al-SIAM18} which contains
both a thorough discussion as well as a wealth of references, and discuss
only some additional papers that are either not mentioned in this paper,
or are not elaborated sufficiently therein.

To give some structure, we divide the discussion into the following topics:
\bit
\item Preliminaries
\item Classes of functions
\item Types of stochastic gradients
\item Nature of convergence
\eit

\subsection{Preliminaries}\label{ssec:30}

We begin with two ``standing'' assumptions, which are standard
in the literature.
These assumptions are assumed to hold in the remainder of the paper.
Note that $J(\cdot)$ denotes the objective function.
\ben
\item[(S1)] $J(\cdot)$ is $\C^1$, and $\gJ(\cdot)$ is globally 
Lipschitz-continuous with constant $L$.
\item[(S2)] $J(\cdot)$ is bounded below.
Thus
\bd
J^* := \inf_{\bth \in \R^d} J(\bth) > - \infty .
\ed
Note that it is \textit{not} assumed that the infimum is actually attained.
\een
Next we present a useful consequence of Assumptions (S1)
and (S2).\footnote{We are
grateful to Reviewer No.\ 2 for suggesting this result and its proof.}
\begin{lemma}\label{lemma:31}
Suppose (S1) holds, and that 
\bd
J^* := \inf_{\bth \in \R^d} J(\bth) > - \infty .
\ed
Then
\be\label{eq:122a}
\nmeusq{ \nabla J(\bth) } \leq 2L [ J(\bth) - J^* ].
\ee
\end{lemma}

\begin{proof}
By applying \cite[Eq.\ (2.4)]{Ber-Tsi-SIAM00} to $J(\bth)$, it follows that,
for every $\bphi , \bth \in \R^d$, we have
\bd
J^* \leq J(\bphi) \leq J(\bth) + \IP{ \nabla J(\bth) }{ \bphi - \bth}
+ \frac{L}{2} \nmeusq{ \bphi - \bth} .
\ed
Now choose $\bphi = \bth - (1/L) \nabla J (\bth)$.
This leads to
\bd
J^* \leq J(\bth) - \frac{1}{L} \nmeusq{ \nabla J(\bth) }
+ \frac{1}{2L} \nmeusq{ \nabla J(\bth) }
= J(\bth) - \frac{1}{2L} \nmeusq{ \nabla J(\bth) } .
\ed
This is the same as \eqref{eq:122}.
This completes the proof.
\end{proof}

\textbf{Remark:} By replacing $J(\bth)$ by $J(\bth) - J^*$,
it can be assumed that
the global infimum of $J(\cdot)$ equals zero, and we do so hereafter. 
In this case, \eqref{eq:122a} can be replaced by
\be\label{eq:122}
\nmeusq{ \nabla J(\bth) } \leq 2L J(\bth) .
\ee

When the function $J(\cdot)$ attains its minimum, the set
\be\label{eq:121}
S_J := \{ \bth : J(\bth) = J^* \}
\ee
is nonempty.
In this case, we define, as usual, the distance
\be\label{eq:121a}
\r(\bth) := \inf_{\bphi \in S_J} \nmeu{\bth - \bphi} .
\ee

In the context of function minimization, a possibly random sequence
of iterations $\{ \bth_t \}$ is generated.
Then we can pose three questions:
\ben
\item[(Q1)] Does $J(\bth_t) \ap 0$ as $\tai$?\footnote{Recall the
assumption that $J^* = 0$.}
\item[(Q2)] Does $\nmeu{\gJ(\bth_t)} \ap 0$ as $\tai$?
\item[(Q3)] Does $\r(\bth_t) \ap 0$ as $\tai$?
\een

In order to address the three questions above, we introduce some
assumptions on $J(\cdot)$.
Some of these assumptions make use of the following concepts.
The first concept
is introduced in \cite{Gladyshev65} but without giving it a name.
The formal definition is given in \cite[Definition 1]{MV-MCSS23}:

\begin{definition}\label{def:Class-B}
A function $\eta : \R_+ \ap \R_+$ is
said to \textbf{belong to Class $\B$} if $\eta(0) = 0$, and in addition,
for arbitrary real numbers $0 < \e \leq M$,
it is true that
\bd
\inf_{\e \leq r \leq M} \eta(r) > 0 .
\ed
\end{definition}

Note $\eta(\cdot)$ is \textit{not} assumed to be monotonic, or even to be
continuous.
However, if $\eta : \R_+ \ap \R_+$ is continuous, then
$\eta(\cdot)$ belongs to Class $\B$ if and only if (i) $\eta(0) = 0$,
and (ii) $\eta(r) > 0$ for all $r > 0$.
Such a function is called a ``class $P$ function'' in 
\cite{Gruene-Kellett14}.
Thus a Class $\B$ function is slightly more general than a function
of Class $P$.

An example of a function of Class $\B$ is given next:
\begin{example}\label{exam:1}
Define a function $\phi: \R_+ \ap \R_+$ by
\bd
\phi(\th) = \left\{ \ba{ll} \th, & \mbox{if } \th \in [0,1] , \\
e^{-(\th-1)}, & \mbox{if } \th > 1 . \ea \right.
\ed
Then $\phi$ belongs to Class $\B$.
A sketch of the function $\phi(\cdot)$ is given in Figure \ref{fig:1}.
Note that, if we were to change the definition to:
\bd
\phi(\th) = \left\{ \ba{ll} \th, & \mbox{if } \th \in [0,1] , \\
2 e^{-(\th-1)}, & \mbox{if } \th > 1 , \ea \right.
\ed
then $\phi(\cdot)$ would be discontinuous at $\th = 1$, but it would
still belong to Class $\B$.
Thus a function need not be continuous to belong to Class $\B$.

\bfig
\bc
\includegraphics[width=60mm]{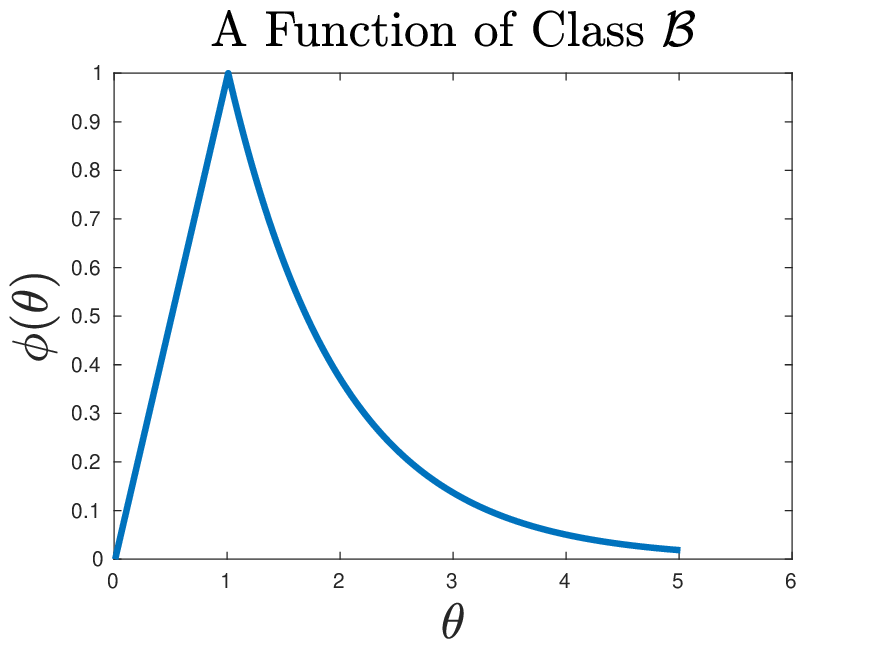}
\ec
\caption{An illustration of a function in Class $\B$}
\label{fig:1}
\efig
\end{example}

\subsection{Classes of Functions}\label{ssec:31}

In this subsection we introduce various classes of functions that are
used in this paper.
Note that different theorems make use of different classes of functions,
which in turn lead to different conclusions.
\ben
\item[(PL)] There exists a constant $K$ such that
\be\label{eq:123}
\nmeusq{\gJ(\bth)} \geq K J(\bth) , \fa \bth \in \R^d .
\ee
\item[(KL')] There exists a function $\psi(\cdot)$ of Class $\B$
such that
\be\label{eq:124}
\nmeu{\gJ(\bth)} \geq \psi(J(\bth)) , \fa \bth \in \R^d .
\ee
\item[(NSC)]
This property consists of the following assumptions, taken together.
\ben
\item The function $J(\cdot)$ attains its infimum.
Therefore the set $S_J$ defined in \eqref{eq:121} is nonempty.
\item
The function $J(\cdot)$ has compact level sets.
For every constant $c \in (0,\infty)$, the level set
\bd
L_J(c) := \{ \bth \in \R^d : J(\bth) \leq c \}
\ed
is compact.
\item
There exists a number $r > 0$ and a continuous function
$\eta : [0,r] \ap \R_+$ such that $\eta(0) = 0$, and
\be\label{eq:125}
\r(\bth) \leq \eta(J(\bth)) , \fa \bth \in L_J(r) .
\ee
\een
It is obvious that, if (NSC) is satisfied, then $J(\bth_t) \ap 0$ as $\tai$
implies that $\r(\bth_t) \ap 0$ as $\tai$.
\een
Next we discuss the significance of these assumptions, as well as the
nomenclature.
As a quick summary, (KL') allows us to conclude that $J(\bth_t)$
and $\gJt$ converge to zero as $\tai$.
These are questions (Q1) and (Q2) above.
Property (PL) allows goes beyond (KL') and not allows us to deduce
convergence, but also bound the \textit{rate} at which convergence takes place.
Finally (NSC) allows us to answer (Q3) above, namely the convergence of
$\bth_t$ to the set of minima $S_J$.

PL stands for the Polyak-Lojasiewicz condition.
In \cite{Polyak-UCMMP63}, Polyak introduced \eqref{eq:123},
and showed that it is sufficient to ensure that iterations converge at
a ``linear'' (or geometric) rate to a global minimum, whether or not
$J(\cdot)$ is convex.
Note that \eqref{eq:123} can also be rewritten as
\bd
\nmeu{\gJ(\bth)} \geq K^{1/2} [J(\bth)]^{1/2} , \fa \bth \in \R^d .
\ed
Suppose $J(\cdot)$ is $R$-strongly convex in the sense of
\cite[Definition 2.1.3]{Nesterov04}, that is, there exists a constant
$R > 0$ such that
\bd
J(\bphi) \geq J(\bth) + \IP{\nabla J(\bth)}{\bphi - \bth}
+ \frac{R}{2} \nmeusq{\bphi - \bth} .
\ed
In this case, $J(\cdot)$ has a unique global minimizer, call it $\bths$.
Again, let us assume that $J^* = J(\bths) = 0$.
Then we can apply \cite[Eq.\ (2.1.24)]{Nesterov04} with $f = J$,
$x = \bths$, $y = \bth$, and $\mu = R$, which gives
\bd
J(\bth) \leq \frac{1}{2R} \nmeusq{\nabla J(\bth)} .
\ed
Thus an $R$-strongly convex function satisfies (PL) with $K = 2R$.
On the other hand, the class (PL) is strictly larger than not just
the class of strongly convex functions, and also contains \textit{some}
nonconvex functions.
For example,
\bd
J(\th) = \th^2 + \sin^2 \th
\ed
is not convex, but satisfies the (PL) property with $K = 1$.
To summarize, any conclusion about (PL) functions would also apply to
strongly convex functions, but not vice versa.

In \cite{Loja63}, Lojasiewicz introduced a more general condition
\be\label{eq:126}
\nmeu{J(\bth)} \geq C [ J(\bth)]^r , \fa \bth \in \R^d ,
\ee
for some constant $C$ and some exponent $r \in [1/2,1)$.
He also showed that \eqref{eq:126} holds for real algebraic varieties
in a neighborhood of critical points.
Note that in the present paper, we use only the Polyak condition
\eqref{eq:123}.

In \cite{Kurdyka98}, Kurdyka proposed a more general inequality than
\eqref{eq:126}, 
namely: There exist a constant $c > 0$ and a function
$v: [0,c) \ap \R$ which is $\C^1$ on $(0,c)$, such that $v'(x) > 0$
for all $x \in (0,c)$, and
\be\label{eq:127}
\nmeu{ \nabla (v \circ J)(\bth) } \geq 1 , \fa \bth \in J^{-1}(0,c) ,
\ee
where (only on this occasion)
$\circ$ denotes the composition of two functions.
By applying the chain rule, one can rewrite \eqref{eq:127} as
\be\label{eq:128}
\nmeu{\gJ(\bth)} \geq [ v'(J(\bth) ]^{-1} .
\ee
In particular, if $v(x) = x^{1-r}$ for some $r \in (0,1)$, then
\eqref{eq:128} becomes \eqref{eq:126} with $C = 1/(1-r)$.
For this reason, \eqref{eq:128} is sometimes referred to as the
Kurdyka-Lojasiewicz (KL) inequality.
See for example \cite{BDLM-TAMS10}.
In our case, we don't require the right side to be a differentiable
function; rather we require only that it be a function of Class $\B$
of $J(\bth)$.
Hence we choose to call this condition as (KL'), to suggest that it is
similar to, but weaker than, the KL condition.

\begin{example}\label{exam:31}
Consider an even function $J: \R \ap \R$ defined by
\bd
J*\th) = \left\{ \ba{ll}
\th^2 + 4 \; \sin^2 \th , & 0 \leq \th \leq 5 , \\
J(5) + 0.5 \; J'(5) \; (1 - \exp(-2(\th-5))) , & \th > 5 , \\
J(-\th) , & \th < 0 .
\ea \right.
\ed
A plot of $J(\th)$ and of $[J'(\th)]^2/J(\th)$ are shown in
Figure \ref{fig:KL}.
From this it can be seen (and it is also readily verified) that,
though the ratio $[J'(\th)]^2/J(\th) \ap 0$ as $\th \ap \infty$,
the ratio is never actually zero.
Thus $[J'(\th)]^2/J(\th)$ is a function of Class $\B$.
As a result, this function satisfies the property (KL').

\bfig
\bc
\includegraphics[height=60mm]{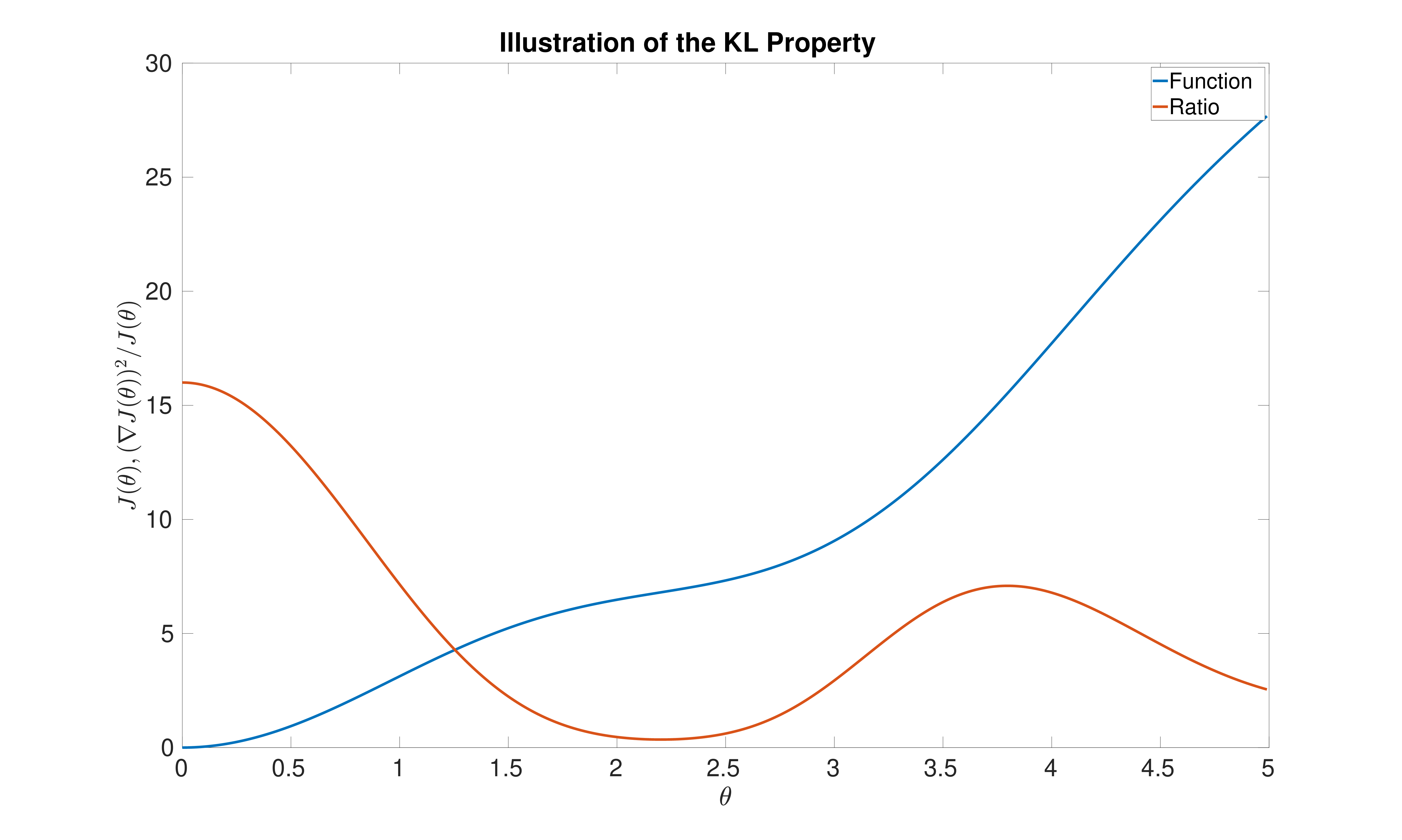}
\ec
\caption{An example of a function of Class (KL')}
\label{fig:KL}
\efig
\end{example}

Finally we come to (NSC), which stands for ``Nearly Strongly Convex.''
If $J(\cdot)$ is $R$-strongly convex with a unique global minimizer
$\bths$, then
\bd
J(\bth) \geq \frac{R}{2} \nmeusq{\bth - \bths} , \fa \bth \in \R^d .
\ed
In this case $S_J$ is the singleton set $\{ \bths \}$,
$\r(\bth) = \nmeu{\bth - \bths}$ (or just $\nmeu{\bth}$ if $\bths = \bz$).
Hence (NSC) holds with $\r(r) = (2r/R)^{1/2}$.
However, in the present paper, we don't assume that there is a unique global
minimizer, nor strong convexity; we just assume (NSC).

Note that, under (PL) or (KL'), $\gJ(\bth) = \bz$ implies that
$J(\bth) = 0$, i.e., that every stationary point is also a global minimum.
Thus any function that satisfies either (PL) or (KL') is ``invex''
as defined in \cite{Hanson-JMAA81}.
See \cite{Karimi-et-al16} for an excellent survey of these topics.

While most of the literature on optimization addresses \textit{convex}
optimization, there are some papers where the convergence of the stochastic
gradient descent method is analyzed under the Kurdyka-Lojasiewicz
condition of \eqref{eq:128}.
See for example \cite{Fontaine-et-al-CoLT21,Fatkhullin-et-al-Neurips22}.
As shown below, in this paper we strengthen the results in the above papers
by proving almost sure convergence, and also replacing (KL) by the weaker
(KL') condition.

\subsection{Types of Stochastic Gradients}\label{ssec:32}

In this subsection, we list some approaches to choosing a stochastic
gradient $\h_{t+1}$ in \eqref{eq:112a}.
It is noteworthy that the phrase ``stochastic gradient'' is used with
two different meanings in the literature.
Both of them are discussed here.

We begin with the following specific type of optimization problem:
Suppose $\X$ is some set, and $\pi$ is some probability measure on $\X$.
Suppose further that $f : \X \times \R^d \ap \R$ is a $\C^1$
function, and define the objective function
\be\label{eq:1211}
J(\bth) := E_{x \sim \pi} [ f(x,\bth) ] = \int_\X f(x,\bth) \; \pi(dx) .
\ee
For the moment, we ignore technicalities about the well-definedness of
the integral.

A typical application would be neural network training.
Suppose $\x \in \R^n$ is the input to the network, $y \in \R$
the desired output with input $\x$ (the label), and $\bth$ is
the set of ``weights'' or adjustable parameters in the network.
A neural network ``architecture'' defines  family of maps
$H(\cdot,\bth) : \R^n \ap \R$ for each $\bth \in \R^d$.
Finally, there is a ``loss function'' $L : \R \times \R \ap \R_+$;
quite often $L(y,z) = |y-z|^2$.
The training data consists of labelled pairs $\{ (\x_i , y_i) \}_{i=1}^m$.
To choose the weight vector optimally, one minimizes
\bd
J(\bth) := \frac{1}{m} \sum_{i=1}^m L(y_i,H(\x_i,\bth)) .
\ed
To put this problem within the framework of \eqref{eq:1211},
define $\X$ to be the finite set $\{ (\x_1,y_1) , \cdots , (\x_m,y_m) \}$,
and choose $\pi$ to be the uniform distribution on $\X$.

Next we discuss three approaches to approximating $\gJ(\bth)$
when $J(\cdot)$ is as in \eqref{eq:1211}.
Further details can be found in \cite[Section 3.3]{Bottou-et-al-SIAM18}.
Under mild conditions, it can be shown that
\be\label{eq:1212}
\gJ(\bth) = E_{x \sim \pi} [ \nabla_{\bth} f(x,\bth) ] .
\ee
In other words, it is permissible to interchange differentiation and
integration in \eqref{eq:1212}.
If $\X$ is a finite set, then this is automatic.

\noindent \textbf{Stochastic Gradient:}
At step $t$, generate an element $x_{t+1} \in \X$ with marginal
distribution $\pi$.
To permit adaptive sampling, it is \textit{not assumed} that $x_{t+1}$ is
independent of the preceding samples $( x_1 , \cdots , x_t)$.
Then the search direction $\h_{t+1}$ is set equal to
\be\label{eq:1213}
\h_{t+1} = \nabla_{\bth}f(x_{t+1},\bth_t) .
\ee

\noindent \textbf{Batch Update:} 
In this case,
\bd
\h_{t+1} = \gJ(\bth_t)
\ed
as computed in \eqref{eq:1212}.
Note that the computation is quite straight-forward if $\X$ is a finite set.

\noindent \textbf{Minibatch Update:}
This approach is intermediate between the above two approaches.
At step $t$, an integer $N_t$ (possibly random) is chosen, and $N_t$
samples $x_j , j \in [N_t]$ are chosen from $\X$.
The analysis is simplest if these samples are drawn independently with
distribution $\pi$, after replacement.
Then
\be\label{eq:1214}
\h_{t+1} = \frac{1}{N_t} \sum_{j=1}^{N_t} \nabla_{\bth} f(x_j,\bth_t) .
\ee
If there are repeated samples, then the corresponding terms are
summed more than once in the above equation.

Until now, we have focused on objective functions of the form \eqref{eq:1211},
and ways to approximate its gradient by random sampling.
Next we discuss approximation methods that apply to general $\C^1$
objective functions.
There are two parts to this:
(i) Constructing approximations to the true gradient, and (ii) selecting
which components of the current guess $\bth_t$ are to be updated at step $t$.
We discuss these two topics in the opposite order.
That is, we begin by discussing
some popular methods of choosing coordinates to be updated,
assuming that the true gradient, corrupted by additive noise, is available.
It will be obvious that the same selection strategies can also be applied
to any stochastic gradient as well.

The first of these methods is referred to as ``Coordinate Gradient Descent''
as in \cite{Wright15} and elsewhere, but also sometimes as
``stochastic gradient,'' thus possibly leading to confusion with
\eqref{eq:1213}.

\noindent \textbf{Coordinate Gradient Descent:}
Suppose that, at step $t$, the current guess is $\bth_t$, and suppose
that the learner has access to a
(possibly noise-corrupted) measurement $\gJ(\bth) + \bxt$.
An index $i \in [d]$ is chosen at random with a uniform probability,
and the search direction is defined as
\be\label{eq:1215}
\h_{t+1} = d \eb_i \circ [ \gJ(\bth_t) + \bxt ] ,
\ee
where $\eb_i$ denotes the $i$-elementary unit vector, and $\circ$ denotes
the Hadamard or componentwise product.\footnote{ If $\a, \b \in \R^d$, then
$\c = \a \circ \b$ belongs to $\R^d$ and is defined via $c_i = a_i b_i$
for all $i$.}
Even if $\bxt \equiv \bz$, $\h_{t+1}$ is still random due to the choice of $i$.
The factor of $d$ is to ensure that the conditional expectation
with respect to $\bth_t$ of $\h_{t+1}$
equals the true gradient $\gJ(\bth_t)$ plus the expectation of $\bxt$.
If this $\h_{t+1}$ is substituted into \eqref{eq:112a}, it is obvious that
only the $i$-th component of $\bth_t$ is updated at time $t$,
and all other components remain the same.

An excellent survey of coordinate gradient descent for convex objective
functions is found in \cite{Wright15}, and some results for
nonconvex objective functions are found in \cite{Bach-et-al-aisats19}.
It is worth pointing out that, in these references and many others, the
error term $\bxt$ is assumed to be zero.
Thus the only source of randomness is the coordinate to be updated.
Much of the detailed analysis carried out in these papers would not be
possible in the presence of measurement errors.

One can also apply this philosophy of updating only
one (possibly randomly chosen)
coordinate at a time to stochastic approximation as in \eqref{eq:111}.
This leads to the update formula
\be\label{eq:1216}
\bth_{t+1}  = \bth_t + \al_t \eb_i \circ [ \f(\bth_t) + \bxt ] .
\ee
In such a case, it is common to refer to this approach as
\textbf{Asynchronous SA} or ASA.
This terminology was apparently introduced in \cite{Tsi-ML94}.
The approach is studied further in \cite{Borkar98}.
In particular, a distinction between using a ``global clock''
and a ``local clock'' for componentwise updating is introduced.

\noindent \textbf{Block Coordinate Gradient Descent:}
A variant of the above is to carry out ``block'' updating.
At each time, a possibly random subset $S_t \seq [d]$ is selected.
Define
\bd
\eb_{S_t} := \sum_{i \in S_t} \eb_i .
\ed
Then the vector $\h_{t+1}$ is defined as
\be\label{eq:1217}
\h_{t+1} := \frac{d}{|S_t|} \eb_{S_t} \circ [ \gJ(\bth_t) + \bxt ] .
\ee               
This implies that, at time $t$, only the components of $\bth_t$,
$i \in S_t$ are updated, and the rest are unchanged.
As above, block updating can also be incorporated in the SA algorithm 
of \eqref{eq:111}, as follows:
\be\label{eq:1218}
\bth_{t+1}  = \bth_t + \balpha_t \circ \eb_{S_t} \circ [ \f(\bth_t) + \bxt ] ,
\ee
where $\balpha_t$ is now a \textit{vector} of step sizes.
Thus, while only those components $i \in S_t$ are updated, different
updated components could have different step sizes.
This topic is not discussed further in this paper.
Instead the reader is referred to \cite{MV-RLK-BASA-arxiv21,MV-RLK-BASA-COT24}
for the latest results.

\noindent \textbf{Gradients Using Only Function Evaluations:}
Next we discuss some approaches to generating approximate gradients that
make use of only function evaluations.
As pointed out above, the first such approach is in \cite{Kief-Wolf-AOMS52},
which is shown above as \eqref{eq:112e}.
It is for the case $d = 1$, and requires two function
evaluations per iteration.
Subsequently Blum \cite{Blum54} presented an approach for the case $d > 1$,
which requires $d+1$ evaluations per iteration.
When $d$ is large, this approach is clearly impractical.
A significant improvement came in\cite{Spall-TAC92}, in which a method
called ``simultaneous perturbation stochastic approximation'' (SPSA)
was introduced, which requires only \textit{two} function evaluations,
irrespective of the dimension $d$.
However, the proof of convergence of SPSA given in \cite{Spall-TAC92} 
requires many assumptions.
These are simplified in \cite{Chen-Dunc-Dunc-TAC97}.
An ``optimal'' version of SPSA is introduced in \cite{Sadegh-Spall-TAC98},
and is described below.

For each index $t+1$, suppose $\D_{t+1,i}, i \in [d]$ are $d$ different and
pairwise independent
\textbf{Rademacher variables}.\footnote{Recall that Rademacher
random variables assume values in $\bp$ and are independent of each other.}
Moreover, suppose that $\D_{t+1,i} , i \in [d]$ are all independent
(not just conditionally independent) of the $\s$-algebra $\F_t$ for each $t$.
Let $\bD_{t+1} \in \bp^d$ denote the vector of Rademacher variables
at time $t+1$.
Then the search direction $\h_{t+1}$ in \eqref{eq:112a} is defined 
componentwise, via
\be\label{eq:1221}
h_{t+1,i} = \frac{[ J(\bth_t + c_t \bD_{t+1}) + \xi_{t+1,i}^+ ]
- [ J(\bth_t - c_t \bD_{t+1}) - \xi_{t+1,i}^- ] } {2 c_t \D_{t+1, i}} ,
\ee
where  $\xi_{t+1,1}^+ , \cdots , \xi_{t+1,d}^+$,
$\xi_{t+1,1}^- , \cdots , \xi_{t+1,d}^-$ represent the measurement errors.
A similar idea is used in \cite{Nesterov-FCM17},
except that the bipolar vector $\bD_{t+1}$ is replaced by a random
Gaussian vector $\boldeta_{t+1}$ in $\R^d$.
As can be seen from the literature, one of the key steps in analyzing
SPSA is to find tail probability estimates of the quantity
$\nmeu{\boldeta_{t+1}}/| \eta_{t+1,i} |$.
If $\boldeta_{t+1}$ is Gaussian, then this ratio can be arbitrarily large,
albeit with small probability.
However, with Rademacher perturbations, the ratio $\nmeu{\bD_{t+1}}/
|\D_{t+1,i}|$ always equals $\sqrt{d}$.
This observation considerably simplifies the analysis.
An excellent survey of this topic can be found in \cite{Li-Xia-Xu-arxiv22},
which discusses other approaches not mentioned here.

The original SPSA envisages only two measurements per iteration, and the
resulting estimate of $\gJ(\bth_t)$ has bias $O(c_t)$ and conditional
variance $O(1/c_t^2)$.
However, it is possible to take more measurements and reduce the bias
of the estimate, while retaining the same bound on the conditional variance.
Specifically, if $k+1$ measurements are taken, then the bias is
$O(c_t^k)$ (which converges to zero more quickly), while the conditional
variance remains as $O(1/c_t^2)$.
See \cite{Pach-Bhat-Pras-arxiv22} and the references therein.

We conclude this subsection by briefly discussing 
\cite{Granichin-ARC02}.
In this paper, a very general framework is proposed that is capable of
handling not only additive measurement errors (as has been the case above),
but also \textit{multiplicative} errors, and others.
Three (closely related) algorithms are proposed in this paper,
out of which only the second one is detailed here, in the interests of brevity.

The set-up is as follows: Suppose $f: \R^n \times \R^d \ap \R$ 
is a $\C^1$ function, and $\pi$ is a (possibly unknown)
probability measure on $\R^p$.
The objective function is as in \eqref{eq:1211}, namely
\bd
J(\bth) = \int_{\R^n} f(\w,\bth) \; \pi(d\w) = E_{\w \sim \pi}[f(\w,\bth)] .
\ed
There is also a probability distribution $P$ on $\R^d$, chosen by the learner,
whose role is to generate an i.i.d.\ sequence of perturbations
$\{ \bD_t \}_{t \geq 1}$.
In addition, there two i.i.d.\ sequences $\{ \w_t^+ \}_{t \geq 0}$,
and $\{ \w_t^- \}_{t \geq 0}$, with distribution $\pi$.
To update the current guess $\bth_t$, one undertakes the following steps.
As with the other derivative-free methods, there are two sequences:
$\{ \al_t \}$ of step sizes, and $\{ c_t \}$ of increments.
At time $t$, the perturbation vector $\bD_{t+1}$ is known, so one can define
\bd
\x_{t+1}^+ = \bth_t + c_t \bD_{t+1} , \quad
\x_{t+1}^- = \bth_t - c_t \bD_{t+1} .
\ed
The measurements available to the learner at time $t$ consist of the pair
\bd
y_{t+1}^+ = f(\w_t^+,\x_{t+1}^+) + \xi_{t+1}^+ , \;\;
y_{t+1}^- = f(\w_t^-,\x_{t+1}^-) + \xi_{t+1}^- ,
\ed
where $\xi_{t+1}^+ , \xi_{t+1}^-$ are measurement errors.
The last step is to define the stochastic gradient $\h_{t+1}$.
This is stated in terms of a sequence of ``kernel functions''
$K_t : \R^d \ap \R^d$ that satisfy, for each $t$
\bd
\int_{\R^d} K_t(\z) \; P(d\z) = \bz , \quad
\int_{\R^d} K_t(\z) \z^\top \; P(d\z) = I_d , \quad
\int_{\R^d} \nmeusq{ K_t(\z) } \; P(d\z) < \infty . 
\ed
With this notation, the stochastic gradient $\h_{t+1}$
is defined as
\bd
\h_{t+1} = \frac{y_{t+1}^+ - y_{t+1}^-} { 2 c_t } K_t(\bD_{t+1}) ,
\ed
with the update rule as in \eqref{eq:112a}, namely
\bd
\bth_{t+1} = \bth_t - \al_t \h_{t+1}  ,
\ed
Note that the choice
\bd
K_t(\z) = (1/z_1 , \cdots , 1/z_d)
\ed
gives the standard Kiefer-Wolfowitz-Blum approach.
However, it is clear that the present scheme offers considerably more
flexibility.

In order to analyze the behavior of the algorithm, it is assumed in
\cite{Granichin-ARC02} that
\ben
\item $J(\cdot)$ is a strongly convex function of $\bth$, and
\item There is a constant $L$ such that $\nabla_{\bth} f(\w,\bth)$
is $L$-Lipschitz continuous for each $\w \in \R^n$.
\een
In particular, Item 1 means that $J(\cdot)$ has a unique global minimizer
$\bths$.
Under these assumptions, \cite[Theorem 1]{Granichin-ARC02} gives
sufficient conditions for $\bth_t$ to converge to $\bths$ in the
mean-squared sense, and almost surely.
The reader is directed to \cite{Granichin-ARC02} for more details.

\subsection{Types and Bounds on the Rates of Convergence}\label{ssec:33}

Since each of these quantities $J(\bth_t), \nmeu{\gJ(\bth_t)}$
and $\r(\bth_t)$ is a random number in general, it is
necessary to specify the nature of the convergence to the desired limit of $0$.
For the most part, the stochastic optimization literature is focused
on convergence in expectation (which in turn implies convergence in
probability).
However, in this paper, we prove the stronger property of almost sure
convergence.
Since any stochastic algorithm generates a single sample path of some
stochastic process, it is advantageous to know that almost all sample paths
converge to the desired limit.
The proofs given here make use of the
Robbins-Siegmund theorem \cite{Robb-Sieg71} and some extensions proved here.
It is worth pointing out that most of the arguments in the survey paper
\cite{Bottou-et-al-SIAM18}, which are used to establish convergence in
expectation, can be readily modified to establish almost sure convergence.

Suppose $\h_{t+1}$ is the stochastic gradient in \eqref{eq:112a}.
For future reference, define
\be\label{eq:1221a}
\z_t = E_t(\h_{t+1}) , \quad
\x_t = \z_t - \gJ(\bth_t) ,\quad
\bzt = \h_{t+1} - \z_t .
\ee
Now we discuss some universal \textit{lower bounds} on the rate of
convergence of the stochastic gradient method, taken from 
\cite{Arjevani-et-al-MP23}, but stated in the notation of the present paper.
The authors study an objective function $J : \R^d \ap \R$ with
a globally Lipschitz-continuous gradient \cite[Eq.\ (3)]{Arjevani-et-al-MP23}.
Further, it is assumed that $\z_t = \gJ(\bth_t)$, and that there is a finite
constant $M$ such that $CV_t(\h_{t+1}) \leq M^2$;
see \cite[Eq.\ (2)]{Arjevani-et-al-MP23}.
Under these assumptions, it is shown that, in the case where $J(\cdot)$
is convex, achieving $\nmeu{\gJ(\bth_t)} \leq \e$ requires
$\OM(\e^{-2})$ iterations in the worst case.
For a nonconvex function, the bound goes up to $\OM(\e^{-4})$.\footnote{There
are some additional technical assumptions which are not repeated here.}
Therefore, if we wish to find a $T$ such that
\bd
\nmeu{\gJ(\bth_t)} \leq \e , \fa t \geq T ,
\ed
then $T = \OM(\e^{-2})$ for convex functions, and $T = \OM(\e^{-4})$ for
nonconvex functions.
We can turn this around to get a bound on the best achievable rate of
convergence.
If $T = \OM(\e^{-k})$, then $\e = O(T^{-1/k})$.
Hence $\nmeu{\gJ(\bth_t)} = O(t^{-1/2})$ if $J(\cdot)$ is convex,
and $\nmeu{\gJ(\bth_t)} = O(t^{-1/4})$ if $J(\cdot)$ is a general
nonconvex function.
One of the contributions of the present paper is to show that the
rate of convergence$\nmeu{\gJ(\bth_t)} = O(t^{-1/2})$ is achieved even
when $J(\cdot)$ is nonconvex, provided that Assumption (PL) is satisfied.


\section{Two New Convergence Theorems for Stochastic Processes}\label{sec:New}

In this section, we prove two new convergence theorems for Stochastic Processes.
The first of these is a slight generalization of the classic theorem of
Robbins-Siegmund \cite{Robb-Sieg71}, while the second gives a recipe
for estimating the rate of convergence of the SA algorithm in
the Robbins-Siegmund setting.


Throughout, all random variables are defined on some underlying
probability space $( \OM, \SI , P )$, and all stochastic
processes are defined on the infinite Cartesian product of this space.

The theorems proved here make use of the following classic
``almost supermartingale theorem'' of Robbins-Siegmund
\cite[Theorem 1]{Robb-Sieg71}.
The result is also proved as \cite[Lemma 2, Section 5.2]{BMP90}.
Also see a recent survey paper \cite[Lemma 4.1]{Fran-Gram22}.
The theorem states the following:

\begin{lemma}\label{lemma:2}
Suppose $\{ z_t \} , \{ f_t \} , \{ g_t \} , \{ h_t \}$ are
stochastic processes taking values in $[0,\infty)$, adapted to some
filtration $\{ \F_t \}$, satisfying
\be\label{eq:511}
E_t( z_{t+1} ) \leq (1 + f_t) z_t + g_t - h_t \as, \fa t ,
\ee
where, as before, $E_t(z_{t+1})$ is a shorthand for $E(z_{t+1} | \F_t )$.
Then, on the set
\bd
\OM_0 := \{ \om \in \OM : \sum_{t=0}^\infty f_t(\om) < \infty \}
\cap \{ \om : \sum_{t=0}^\infty g_t(\om) < \infty \} ,
\ed
we have that $\lim_{\tai} z_t$ exists, and in addition,
$\sum_{t=0}^\infty h_t(\om) < \infty$.
In particular, if $P(\OM_0) = 1$, then $\{ z_t \}$ is bounded
almost surely, and $\sum_{t=0}^\infty h_t(\om) < \infty$ almost surely.
\end{lemma}


Now we present our first convergence theorem, which is
an extension of Lemma \ref{lemma:2}.
Though the proof is straight-forward, we will see that it is
a useful tool to establish the convergence of stochastic gradient
methods for nonconvex functions.

\begin{theorem}\label{thm:51}
Suppose $\{ z_t \} , \{ f_t \} , \{ g_t \} , \{ h_t \}, \{ \al_t \}$ are
$[0,\infty)$-valued stochastic processes
defined on some probability space $(\OM,\SI,P)$, and
adapted to some filtration $\{ \F_t \}$.
Suppose further that
\be\label{eq:512}
E_t(z_{t+1} ) \leq (1 + f_t) z_t + g_t - \al_t h_t \as, \fa t .
\ee
Define
\be\label{eq:513}
\OM_0 := \{ \om \in \OM : \sum_{t=0}^\infty f_t(\om) < \infty \mbox{ and }
\sum_{t=0}^\infty g_t(\om) < \infty \} ,
\ee
\be\label{eq:513a}
\OM_1 := \{ \om \in \OM : \sum_{t=0}^\infty \al_t(\om) = \infty \} .
\ee
Then
\ben
\item
Suppose that $P(\OM_0) = 1$.
Then the sequence $\{ z_t \}$ is bounded almost surely, and
there exists a random variable $W$ defined on $(\OM,\SI,P)$ such that
$z_t(\om) \ap W(\om)$ almost surely.
\item
Suppose that, in addition to $P(\OM_0) = 1$, it is also true that
$P(\OM_1) = 1$.
Then
\be\label{eq:514}
\liminf_{\tai} h_t(\om) = 0  \fa \om \in \OM_0 \cap \OM_1 .
\ee
Further, suppose there exists a function $\eta(\cdot)$ of Class $\B$
such that $h_t(\om) \geq \eta(z_t(\om))$ for all $\om \in \OM_0$.
Then $z_t(\om) \ap 0$ as $\tai$ for all $\om \in \OM_0$.
\een
\end{theorem}

\begin{proof}
By Lemma \ref{lemma:2}, there exists
a random variable $W$ such that $z_t(\om) \ap W(\om)$ as $\tai$
for almost all $\om \in \OM_0$.
This implies that $z_t$ is bounded almost surely.
This is Item 1.

Next we prove item 2.
Again from Lemma \ref{lemma:2},
\bd
\sum_{t=0}^\infty \al_t(\om) h_t(\om) < \infty , \fa \om \in \OM_0 .
\ed
Now, by definition
\bd
\sum_{t=0}^\infty \al_t(\om) = \infty , \fa \om \in \OM_0 \cap \OM_1 .
\ed
Therefore \eqref{eq:514} follows.
Next, suppose that, for some $\om \in \OM_0 \cap \OM_1$, we have that
$W(\om) > 0$, say $W(\om) =: 2 \e > 0$.
Choose a time $T$ such that $z_t(\om) \geq \e$ for all $t \geq T$.
Also, by Item 1,
\bd
M := \sup_{t \geq T} z_t(\om) < \infty .
\ed
Since $z_t(\om) \ap 2 \e$ as $\tai$, it is clear that $M \geq 2\e$.
Next, since $\eta(\cdot)$ belongs to Class $\B$, it follows that
\bd
c := \inf_{\e \leq r \leq M} \eta(r) > 0 .
\ed
So, for $t \geq T$, we have that
\bd
h_t(\om) \geq \eta(z_t(\om)) \geq c .
\ed
Now, if we discard all terms for $t < T$, we get
\bd
\sum_{t=T}^\infty \al_t(\om) h_t(\om) < \infty , \fa \om \in \OM_0 , \quad
\sum_{t=T}^\infty \al_t(\om) = \infty ,
h_t(\om) \geq c > 0 ,
\ed
which is clearly a contradiction.
Therefore the set of $\om \in \OM_0 \cap \OM_1$ for which $W(\om) > 0$ has zero measure
within $\OM_0 \cap \OM_1$.
In other words, $z_t(\om) \ap 0$ for (almost) all $\om \in \OM_0 \cap \OM_1$.
This is Item 2.
This completes the proof.
\end{proof}

Theorem \ref{thm:51} above shows only that $z_t$
converges to $0$ almost surely on sample paths in $\OM_0 \cap \OM_1$.
In this paper,  we are interested not only in the convergence
of various algorithms, but also on the \textit{rate} of convergence.
With this in mind, we now state and prove an extension of
Theorem \ref{thm:51} that provides such an estimate on rates.
For the purposes of this paper, we use the
following definition.

\begin{definition}\label{def:order}
Suppose $\{ Y_t \}$ is a stochastic process, and $\{ f_t \}$
is a sequence of positive numbers.
We say that
\ben
\item $Y_t = O(f_t)$ if $\{ Y_t / f_t \}$ is bounded almost surely.
\item $Y_t = \OM(f_t)$ if $Y_t$ is positive almost surely, and
$\{ f_t/Y_t \}$ is bounded almost surely.
\item $Y_t = \Th(f_t)$ if $Y_t$ is both $O(f_t)$ and $\OM(f_t)$.
\item $Y_t = o(f_t)$ if $Y_t /f_t \ap 0$ almost surely as $\tai$.
\een
\end{definition}

The next theorem is a modification of Theorem \ref{thm:51} that
provides bounds on the \textit{rate} of convergence.
Hence forth, all assumptions hold ``almost surely,'' that is,
along almost all sample paths.
Hence we drop this modifier hereafter, it being implicitly understood.

\begin{theorem}\label{thm:52}
Suppose $\{ z_t \} , \{ f_t \} , \{ g_t \} , \{ \al_t \}$ are
stochastic processes defined on some probability space $(\OM,\SI,P)$,
taking values in $[0,\infty)$, adapted to some
filtration $\{ \F_t \}$.
Suppose further that
\be\label{eq:523a}
E_t(z_{t+1} ) \leq (1 + f_t) z_t + g_t - \al_t z_t , \fa t ,
\ee
where
\bd
\sum_{t=0}^\infty f_t(\om) < \infty , \quad 
\sum_{t=0}^\infty g_t(\om) < \infty , \quad 
\sum_{t=0}^\infty \al_t(\om) = \infty .
\ed
Then $z_t = o(t^{-\l})$ for every $\l \in (0,1)$ such that
(i) there exists an integer $T \geq 1$ such that
\be\label{eq:523b}
\al_t(\om) - \l t^{-1} \geq 0 , \fa t \geq T ,
\ee
and in addition (ii)
\be\label{eq:523c}
\sum_{t=1}^\infty (t+1)^\l g_t(\om) < \infty , \quad 
\sum_{t=1}^\infty [ \al_t(\om) - \l t^{-1} ] = \infty .
\ee
\end{theorem}

\textbf{Remark:}
In \cite{Liu-Yuan-arxiv22}, it is assumed that $\al_t = \Th(t^{-(1-\th)})$.
No such assumption is made here.

\begin{proof}
Over the interval $(0,\infty)$,
the map $t \mapsto t^\l$ is concave for $\l \in (0,1)$.
It follows from the ``graph below the tangent'' property that
\be\label{eq:524}
(t+1)^\l \leq t^\l + \l t^{\l - 1} , \fa t > 0 .
\ee
This is the same as \cite[Eq.\ (25)]{Liu-Yuan-arxiv22} with the
substitution $\l = 1 - \e$.
Now a ready consequence of \eqref{eq:524} is
\bd
1 \leq \left( \frac{ t+1 }{t} \right)^\l \leq 1 + \l t^{-1} ,
\fa t > 0 .
\ed

Now we follow the suggestion of \cite[Lemma 1]{Liu-Yuan-arxiv22}
by recasting \eqref{eq:523a} in terms of $t^\l z_t$.
We can study\eqref{eq:523a} starting at time $T$, where $T$ is defined
in \eqref{eq:523b}.
If we multiply both sides of \eqref{eq:523a} by $(t+1)^\l$, 
and divide by $t^\l$ where appropriate, we get
\bd
E_t( (t+1)^\l z_{t+1}) \leq (1 + f_t)
\left( \frac{ t+1 }{t} \right)^\l t^\l z_t
+ (t+1)^\l g_t - \al_t \left( \frac{ t+1 }{t} \right)^\l t^\l z_t ,
\fa t \geq T .
\ed
Now we observe that, for all $t \geq T$, we have
\bd
- \al_t \left( \frac{ t+1 }{t} \right)^\l \leq - \al_t ,
\ed
\bd
(1 + f_t) \left( \frac{ t+1 }{t} \right)^\l 
\leq (1 + f_t) ( 1 + \l t^{-1} )
= 1 + f_t ( 1 + \l t^{-1} ) + \l t^{-1} .
\ed
If we now define the modified quantity $\zb_t = t^\l z_t$,
then the above bound can be rewritten as
\be\label{eq:525}
E_t(\zb_{t+1}) \leq [ 1 + f_t ( 1 + \l t^{-1} ) ] \zb_t
+ (t+1)^\l g_t - ( \al_t - \l t^{-1} ) \zb_t 
\fa t \geq T .
\ee
Since $1+\l t^{-1}$ is bounded over $t \geq T$,
it is obvious that
\bd
\sum_{t=T}^\infty f_t < \infty \imp \sum_{t=T}^\infty f_t 
( 1 + \l t^{-1} ) < \infty .
\ed
Moreover, by assumption,
\bd
\al_t - \l t^{-1} \geq 0 , \fa t \geq T .
\ed
Therefore we can apply Theorem \ref{thm:51}
to \eqref{eq:523b}, with $\eta(r) = r$, and deduce that $\zb_t \ap 0$
as $\tai$.
This is equivalent to $z_t = o(t^{-\l})$.
This completes the proof.
\end{proof}

\section{Convergence of Stochastic Gradient Descent}\label{sec:SGD}

In this section, we analyze the convergence of the Stochastic
Gradient Descent (SGD) algorithm of \eqref{eq:112a},
reproduced here for the convenience of the reader:
\be\label{eq:311}
\bth_{t+1} = \bth_t - \al_t \h_{t+1} ,
\ee
where $\h_{t+1}$ is a stochastic gradient.
As in \eqref{eq:1221a}, let us define
\be\label{eq:312}
\z_t = E_t(\h_{t+1}) , \quad
\x_t = \z_t - \gJ(\bth_t) ,\quad
\bzt = \h_{t+1} - \z_t .
\ee
The last equation in \eqref{eq:312} implies that $E_t(\bzt) = \bz$.
Therefore 
\be\label{eq:312a}
E_t( \nmeusq{\h_{t+1}} ) = \nmeusq{\z_t} + E_t \nmeusq{\bzt} .
\ee

In order to analyze the convergence of \eqref{eq:311},
we recall the standing assumptions on $J(\cdot)$, namely:
\ben
\item[(S1)] $J(\cdot)$ is $\C^1$, and $\gJ(\cdot)$ is globally
Lipschitz-continuous with constant $L$.                        
\item[(S2)] $J(\cdot)$ is bounded below.
Thus
\bd
J^* := \inf_{\bth \in \R^d} J(\bth) > - \infty .
\ed
Note that it is \textit{not} assumed that the infimum is actually attained.
\een
Recall from \eqref{eq:122a} that, as a consequence of these assumptions,
it follows that
\be\label{eq:122b}
\nmeusq{ \nabla J(\bth) } \leq 2L J(\bth) , \fa \bth \in \R^d .
\ee

We also recall the three properties to be 
used in proving various theorems, namely
\ben
\item[(PL)] There exists a constant $K$ such that
\bd
\nmeusq{\gJ(\bth)} \geq K J(\bth) , \fa \bth \in \R^d .
\ed
\item[(KL')] There exists a function $\psi(\cdot)$ of Class $\B$
such that
\bd
\nmeu{\gJ(\bth)} \geq \psi(J(\bth)) , \fa \bth \in \R^d .
\ed
\item[(NSC)]
This property consists of the following assumptions, taken together.
\ben
\item The function $J(\cdot)$ attains its infimum.
Therefore the set $S_J$ defined in \eqref{eq:121} is nonempty.
\item
The function $J(\cdot)$ has compact level sets.
For every constant $c \in (0,\infty)$, the level set
\bd
L_J(c) := \{ \bth \in \R^d : J(\bth) \leq c \}
\ed
is compact.
\item
There exists a number $r > 0$ and a continuous function 
$\eta : [0,r] \ap \R_+$ such that $\eta(0) = 0$, and
\bd
\r(\bth) \leq \eta(J(\bth)) , \fa \bth \in L_J(r) .
\ed
\een
\een

Until now, we have just specified properties of the objective
function $J(\cdot)$.
Next, we introduce two key assumptions about the nature of the stochastic
gradient $\h_{t+1}$.

\textbf{Assumption:}
Recall the quantities $\x_t$ and $\bzt$ defined in 
\eqref{eq:312}.
The assumption is that
there exist sequences of constants $\{ \mu_t \}$ and $\{ M_t \}$ such that
\be\label{eq:313}
\nmeu{\x_t} \leq \mu_t [ 1 + \nmeu{\gJt} ] , \fa \bth_t \in \R^d , \fa t ,
\ee
\be\label{eq:314}
E_t ( \nmeusq{\bzt}) \leq M_t^2 [ 1 + J(\bth_t) ] ,
\fa \bth_t \in \R^d , \fa t .
\ee

Now we compare and contrast the significance of these assumptions
with those elsewhere in the literature.
\ben
\item
Note that \eqref{eq:313} permits the stochastic gradient to be a
\textit{biased} estimate of $\gJt$.
This by itself is not unusual.
In several papers, assumptions of the form \eqref{eq:313} occur, but
\textit{without the $\nmeu{\gJt}$ term}.
We now give an example of a situation where the presence of this term 
arises naturally.
Consider the ``Coordinate Gradient Descent'' algorithm described
in \eqref{eq:1215}.
In the traditional approach, every coordinate is sampled uniformly 
at random, which explains the presence of the factor $d$ in the equation.
Now consider an ``off-policy'' type of coordinate sampling, in which,
at time $t$, the coordinates are sampled with a probability distribution
$\bphi_t$, which \textit{need not equal} the uniform distribution.
However, $\bphi_t \ap \u_d$ as $\tai$, where $\u_d$ is the uniform
distribution on a set of $d$ elements.
To analyze this case, let $I_t$ denote the coordinate chosen to be updated
at time $t$.
Then
\bd
I_t = i \wp \; \; \phi_{t,i} .
\ed
Hence the stochastic gradient can be computed as
\bd
\h_{t+1} = d [\gJt] \circ \eb_{I_t} \wp \; \; \phi_{t,i} ,
\ed
To estimate the quantity $\nmeu{\x_t}$ where $\x_t = E_t(\h_{t+1}) - \gJt$,
we use the notation $g_i$ for $[\gJt]_i$, for brevity.
Then
\bd
[ \h_{t+1} - \gJt ]_i
= \left\{ \ba{ll} (d-1) g_i , & \wp \; \; \phi_{t,i} , \\
-g_i, & \wp \; \; \phi_{t,j} , j \neq i . \ea \right.
\ed
Therefore, with $\x_t = E_t(\h_{t+1} - \gJt)$ as earlier, we have that
\begin{eqnarray*}
x_{t,i} & = & (d-1) g_i \phi_{t,i} - \sum_{j \neq i} g_i \phi_{t,j} 
= d g_i \phi_{t,i} - g_i \sum_{j=1}^d \phi_{t,j} \\
& = & ( d \phi_{t,i} - 1 ) g_i = d ( \phi_{t,i} - u_i ) g_i ,
\end{eqnarray*}
where $u_i = 1/d$ is the $i$-th component of the uniform distribution
(for each $i$).
Summing over $i$ leads to
\begin{eqnarray*}
\nmm{\x_t}_1 & = & d \sum_{i=1}^d | ( \phi_{t,i} - u_i ) | \cdot | g_i | \\
& \leq & d \nmm{\bphi_t - \u_d}_1 \nmi{\gJt}  ,
\end{eqnarray*}
where $\nmm{\bphi_t - \u_d}_1$ denotes the $\ell_1$ distance
between $\bphi_t$ and $\u_d$.
Next, after observing that $\nmi{\v} \leq \nmeu{\v} \leq \nmm{\v}_1$,
we arrive finally at
\bd
\nmeu{\x_t} \leq d \nmm{\bphi_t - \u}_1 \nmeu{\gJt} ,
\ed
which is a special case of \eqref{eq:313}.
Note that, when the ``off-policy'' sampling probability distribution is
not the uniform distribution, the presence of the term $\nmeu{\gJt}$
in \eqref{eq:313} is unavoidable.
\item
Next we discuss \eqref{eq:314}.
One can compare \eqref{eq:314} with the so-called Expected
Smoothness condition proposed as Assumption 2
in \cite{Khaled-Rich-arxiv20}, namely
\be\label{eq:314a}
E_t( \nmeusq{\h_{t+1}}) \leq 2 A J(\bth_t) + B \nmeusq{\gJt} + C ,
\ee
for suitable constants $A, B, C$.
This is proposed as ``the weakest assumption'' for analyzing
the convergence of SGD for nonconvex functions.
If $J(\cdot)$ satisfies Assumptions (S1) and (S2), then we can apply
Lemma \ref{lemma:31}.
As a result, the term $B\nmeusq{\gJt}$
can be bounded by $2BL J(\bth_t)$, resulting in
\be\label{eq:314b}
E_t( \nmeusq{\h_{t+1}}) \leq 2 (A + BL) J(\bth_t) + C
\leq M(1 + J(\bth_t) ) ,
\ee
where 
\bd
M =  \max \{2 (A + BL) , C \} .
\ed
Thus, for functions $J(\cdot)$ satisfying Assumptions (S1) and (S2), the present
assumption \eqref{eq:314} is weaker than \eqref{eq:314a}.
Also, the various constants in \eqref{eq:314a} are bounded with respect to $t$,
whereas in \eqref{eq:314}, the bound $M_t$ is allowed to be unbounded
with respect to $t$.
As shown long ago in \cite{Kief-Wolf-AOMS52}, permitting the variance
to be unbounded with time is an essential feature in analyzing SGD
based on function evaluations alone.
\een

With these assumptions, we state the first convergence result,
which \textit{does not have} any conclusions about the rate of convergence.
As always, these bounds and conclusions hold almost surely.

\begin{theorem}\label{thm:61}
Suppose the objective function $J(\cdot)$ satisfies the standing assumptions
(S1) and (S2), and that the stochastic gradient $\h_{t+1}$ satisfies
\eqref{eq:313} and \eqref{eq:314}.
With these assumptions, we have the following conclusions;
\ben
\item Suppose
\be\label{eq:315}
\sum_{t=0}^\infty \al_t^2 < \infty , \quad 
\sum_{t=0}^\infty \al_t \mu_t  < \infty , \quad 
\sum_{t=0}^\infty \al_t^2 M_t^2 < \infty .
\ee
Then $\{ \gJ(\bth_t) \}$ and $\{ J(\bth_t) \}$ are bounded, and in addition,
$J(\bth_t)$ converges to some random variable as $\tai$.
\item If in addition $J(\cdot)$ satisfies (KL'), and
\be\label{eq:316}
\sum_{t=0}^\infty \al_t = \infty ,
\ee 
then $J(\bth_t) \ap 0$ and $\gJ(\bth_t) \ap \bz$ as $\tai$.
\item Suppose that in addition to (KL'), $J(\cdot)$ also satisfies (NSC),
and that \eqref{eq:315} and \eqref{eq:316} both hold.
Then $\r(\bth_t) \ap 0$ as $\tai$.
\een
\end{theorem}

\textbf{Remarks}
\ben
\item Note that if $\mu_t = 0$ for all $t$ (unbiased measurements),
and $M_t^2$ is bounded, then the second condition is \eqref{eq:315} is
automatically satisfied, and the third condition is implied by the first.
In such a case, $\al_t$ is automatically bounded.
Thus \eqref{eq:315} become the first part of the Robbins-Monro conditions
in \eqref{eq:112d}.
Under these conditions, the two sequences
$\{ \gJ(\bth_t) \}$ and $\{ J(\bth_t) \}$ are bounded.
This by itself is not enough to ensure that the iterations $\{ \bth_t \}$
are bounded.
That can be inferred if the objective function $J(\cdot)$ is ``coercive,''
that is, has compact level sets.
\item In the literature, various conditions are imposed on $E_t(\z_t)$
(the conditional mean of the stochastic gradient) and $E_t(\nmeusq{\bzt})$
(the conditional variance of the stochastic gradient.
Equation \eqref{eq:315} includes these as special cases.
\item If in addition to \eqref{eq:315}, the step size sequence
diverges, and $J(\cdot)$ satisfies the property (KL'), then the sequences
$\{ \gJ(\bth_t) \}$ and $\{ J(\bth_t) \}$ converge to the desired
limits of $\bz$ and $0$ respectively.
However, this is not sufficient to infer that the iterations $\bth_t$
converge to the set of global minima.
The addition of property (NSC) allows us to infer that $\bth_t \ap \bz$
as $\tai$.
\item As the proof of the theorem (below) shows, $J(\bth_t) \ap 0$ and
(NSC) together imply that $\bth_t$ converges to the set of global minima.
The mechanism used to infer that $J(\bth_t) \ap 0$ does not matter.
\een

\begin{proof}
The proof is based on Theorem \ref{thm:51}.
It follows from applying  \cite[Eq.\ (2.4)]{Ber-Tsi-SIAM00} to
\eqref{eq:112a} that
\bd
J(\bth_{t+1}) \leq J(\bth_t) - \al_t \IP{\gJt}{\h_{t+1}} + \frac{\al_t^2 L}{2}
\nmeusq{\h_{t+1} } .
\ed
Applying the operator $E_t$ to both sides, using the definitions in
\eqref{eq:312}, and applying \eqref{eq:312a}, gives
\be\label{eq:317}
E_t(J(\bth_{t+1})) \leq J(\bth_t) - \al_t \IP{ \gJt}{\z_t} 
+ \frac{\al_t^2 L}{2} [ \nmeusq{\z_t} + E_t ( \nmeusq{\bzt} ) .
\ee
We will bound each term separately, repeatedly using \eqref{eq:313},
\eqref{eq:314}, Schwarz' inequality, and the obvious inequality
\bd
2a \leq 1 + a^2 , \fa a \in \R .
\ed
First,
\begin{eqnarray*}
\IP{ \gJt}{\z_t} & = & \nmeusq{\gJt} + \IP{\gJt}{\x_t} \\
& \geq & \nmeusq{\gJt} - \nmeu{\gJt} \cdot \nmeu{\x_t} .
\end{eqnarray*}
Now
\beq
\nmeu{\gJt} \cdot \nmeu{\x_t} & \leq & \mu_t \nmeu{\gJt} [ 1 + \nmeu{\gJt} ] 
\nonumber \\
& = & \mu_t \nmeu{\gJt} + \mu_t \nmeusq{\gJt} \nonumber \\
& \leq & 0.5 \mu_t + 1.5 \mu_t  \nmeusq{\gJt} ] \label{eq:317a} \\
& \leq & \mu_t + 2 \mu_t \nmeusq{\gJt} 
\leq \mu_t + 4 \mu_t L J(\bth_t) . \label{eq:318}
\eeq
In the last equation
we have replaced $0.5$ by $1$ just to avoid dealing with fractions,
and have also used \eqref{eq:122b}.
Hence
\begin{eqnarray*}
- \al_t \IP{ \gJt}{\z_t} & \leq & - \al_t \nmeusq{\gJt}
+ \al_t \nmeu{\gJt} \cdot \nmeu{\x_t} \\
& \leq & - \al_t \nmeusq{\gJt} + \al_t \mu_t + 4 \al_t \mu_t L J(\bth_t) .
\end{eqnarray*}
Next,
\begin{eqnarray*}
\nmeusq{\z_t} & \leq & \nmeusq{\gJt} + 2 \nmeu{\gJt} \cdot {\nmeu{\x_t} }
+ \nmeusq{\x_t} \\
& \leq & \nmeusq{\gJt} + \mu_t + 3 \mu_t \nmeusq{\gJt}
+ \nmeusq{\x_t} \\
& \leq & \mu_t + 2L (1+3 \mu_t) J(\bth_t) + \nmeusq{\x_t} .
\end{eqnarray*}
Note that here we use the tighter estimate from \eqref{eq:317a}.
Next,
\begin{eqnarray*}
\nmeusq{\x_t}  & \leq & \mu_t^2 [ 1 + \nmeu{\gJt} ]^2 
= \mu_t^2 [ 1 + 2 \nmeu{\gJt} + \nmeusq{\gJt} ] \\
& \leq & 2 \mu_t^2 [ 1 + \nmeusq{\gJt} ] \leq 2 \mu_t^2 [ 1 + 2L J(\bth_t) ] .
\end{eqnarray*}
Substituting into the above gives the bound
\bd
\nmeusq{\z_t} \leq \mu_t + 2 \mu_t^2 + 2L ( 1 + 3 \mu_t + 2 \mu_t^2 ) J(\bth_t) .
\ed
Finally, by assumption \eqref{eq:314},
\bd
E_t( \nmeusq{\bzt}) 
\leq M_t^2 [ 1 + 2L J(\bth_t) ] .
\ed
Substituting these bounds into \eqref{eq:317} gives a bound to which
Theorem \ref{thm:51} can be applied, namely:
\be\label{eq:319}
E_t( J(\bth_{t+1} ) ) \leq (1+f_t) J(\bth_t) + g_t 
- \al_t \nmeusq{\gJt} ,
\ee
where
\be\label{eq:3110}
f_t = 2L [ 2 \al_t \mu_t + \frac{L}{2} \al_t^2 ( 1 + 3 \mu_t + 2 \mu_t^2 ) 
+ \al_t^2 M_t^2 ] ,
\ee
\be\label{eq:3111}
g_t = \al_t \mu_t + 
\frac{L}{2} \al_t^2 ( \mu_t + 2 \mu_t^2 + M_t^2 ) .
\ee
Now it is straight-forward to verify that the conditions in \eqref{eq:315}
suffice to establish that both sequences $\{ f_t \}$ and
$\{ g_t \}$ are summable.
There are five different terms occuring in \eqref{eq:3110} and
\eqref{eq:3111}, namely
\bd
\al_t^2 , \;
\al_t \mu_t , \; \al_t^2 \mu_t ,\;  \al_t^2 \mu_t^2 ,\;  \al_t^2 M_t^2 .
\ed
Now \eqref{eq:315} states that $\{ \al_t^2 \}$,
$\{ \al_t \mu_t \}$ and $\{ \al_t^2 M_t^2 \}$ are summable.
The first condition implies that $\al_t$ is bounded, 
which implies that $\{ \al_t^2 \mu_t \}$ is also summable.
Finally, since every summable sequence is also square-summable
($\ell_1$ is a subset of $\ell_2)$, $\{ \al_t^2 \mu_t^2 \}$
is also summable.
Since all the conditions needed to apply Item 1 of Theorem \ref{thm:51} hold,
it follows that $\{ J(\bth_t) \}$ is bounded and
converges to some random variable.
Now Lemma \ref{lemma:31} implies that $\gJt$ is also bounded.
This establishes the Item 1 of the theorem.

To prove Item 2, note that if property (KL') holds, then Item 2 of
Theorem \ref{thm:51} applies, and $J(\bth_t) \ap 0$ as $\tai$.

Finally, Item 3 is a ready consequence of $J(\bth_t) \ap 0$ and property (NSC).
This completes the proof.
\end{proof}

Next we strengthen Assumption (KL') to (PL), and prove an estimate
for the \textit{rate} of convergence.

\begin{theorem}\label{thm:62}
Let various symbols be as in Theorem \ref{thm:61}.
Suppose $J(\cdot)$ satisfies the standing assumptions (S1) and (S2)
and also property (PL),
and that \eqref{eq:315} and \eqref{eq:316} hold.
Further, suppose there exist constants $\g > 0$ and $\d \geq 0$ such
that
\bd
\mu_t = O(t^{-\g}), \quad
M_t = O(t^\d) , \fa t \geq 1 ,
\ed
where we take $\g = 1$ if $\mu_t = 0$ for all sufficiently large $t$,
and $\d = 0$ if $M_t$ is bounded.
Choose the step-size sequence $\{ \al_t \}$ as
$O(t^{-(1-\phi)})$ and $\OM(t^{-(1-C)})$
where $\phi$ and $C$ are chosen to satisfy
\bd
0 < \phi < \min \{ 0.5 - \d , \g \} , \quad
C \in (0,\phi] .
\ed
Define
\be\label{eq:3112}
\nu := \min \{ 1 - 2( \phi + \d) , \g - \phi \} .
\ee
Then $\nmeusq{\gJt} = o(t^{-\l})$ and $J(\bth_t) = o(t^{-\l})$
for every $\l \in (0,\nu)$.
In particular, by choosing $\phi$ very small, it follows that
$\nmeusq{\gJt} = o(t^{-\l})$ and $J(\bth_t) = o(t^{-\l})$ whenever
\be\label{eq:3113}
\l < \min \{ 1 - 2 \d , \g \} .
\ee
\end{theorem}

\begin{proof}
Recall the bound \eqref{eq:319} and the definitions of $f_t, g_t$
from \eqref{eq:3110} and \eqref{eq:3111} respectively.
Replacing the property (KL') by property (PL) allows us to
replace the term $-\al_t \nmeusq{\gJt}$ in \eqref{eq:319} 
by $-\al_t K J(\bth_t)$.
This makes Theorem \ref{thm:52} applicable to the resulting bound.
Under the stated hypotheses, it readily follows that
\bd
\al_t^2 = O( t^{-2 + 2 \phi } ) ,
\al_t^2 M_t^2  = O( t^{-2 + 2(\phi +\d)} ) ,
\al_t \mu_t = O(t^{-1 + \phi - \g}) .
\ed
Now define $\nu$ as in \eqref{eq:3112}.
Then each of the above three terms is $O(t^{-(1+\nu)})$, while
both $\{ \al_t^2 \mu_t^2 \}$ and $\{ \al_t^2 \mu_t \}$ decay even faster.
Hence, with $\nu$ defined as in \eqref{eq:3112},
\bd
f_t , g_t = O(t^{-(1 + \nu)}) ,
\ed
and both sequences are summable.

Now we are in a position to apply Theorem \ref{thm:52}.
We can conclude that $J(\bth_t) = o(t^{-\l})$ whenever
$2 \al_t - \l t^{-1} \geq 0$ for sufficiently large $t$, and
\bd
\{ (t+1)^\l g_t \} \in \ell_1 ,
\ed
\be\label{eq:3114}
\sum_{t=1}^\infty [ 2 \al_t - \l t^{-1} ] = \infty .
\ee
Now observe that $2 \al_t = \OM(t^{-(1-C)})$, and $C > 0$.
Choose a contant $D$ such that $2 \al_t \geq D t^{-(1-C)}$ for sufficiently
large $t$.
Then, whatever be the value of $\l$, it is clear that 
\bd
D t^{-(1-C)} - \l t^{-1} \geq 0 
\ed
for sufficiently large $t$.
Also, since $C > 0$, it is evident that $\al_t$ decays more slowly than
$\l t^{-1}$.
Hence \eqref{eq:3114} is satisfied.
Thus the last step of the proof is to determine conditions under which
$\{ (t+1)^\l g_t \} \in \ell_1$.
Since $g_t = O(t^{-(1+\nu)})$, it follows that $(t+1)^\l g_t 
= O(t^{-(1+\nu-\l)})$, which is summable if $\l < \nu$.
Hence it follows that $J(\bth_t) = o(t^{-\l})$ whenever $\l < \nu$.

To prove the last statement, observe that, while there is an upper bound
on $\phi$, namely $\min \{ 0.5 - \d , \g \}$, there is no lower bound.
So we can choose $\phi = \e$, a very small number.
This leads to
\bd
\l < \nu = \min \{ 1 - 2 \d - 2 \e , \g - \e \} .
\ed
But since $\e$ can be made arbitrarily small, this translates to
\eqref{eq:3113}.
This completes the proof.
\end{proof}

\begin{corollary}\label{coro:61}
Suppose all hypotheses of Theorem \ref{thm:62} hold.
In particular, if $\mu_t = 0$ for all large enough $t$ in \eqref{eq:313},
and $M_t$ in \eqref{eq:314} is bounded with respect to $t$, then
$\nmeusq{\gJt} = o(t^{-\l})$ and $J(\bth_t) = o(t^{-\l})$ for all $\l < 1$.
\end{corollary}

The proof is immediate from Theorem \ref{thm:62}.
With $\mu_t = 0$, one can take $\g = 1$, and with $M_t$ being bounded, one
can take $\d = 0$.
Substituting these into \eqref{eq:3113} leads to the desired conclusion.

\textbf{Remark:}
It is worthwhile to compare the content of Corollary \ref{coro:61}
with the bounds from \cite{Arjevani-et-al-MP23},
as summarized in Section \ref{ssec:33}.
In that paper, it is assumed that $\z_t = \gJ(\bth_t)$, and that there
is a finite constant $M$ such that $CV_t(\h_{t+1}) \leq M^2$;
see \cite[Eq.\ (2)]{Arjevani-et-al-MP23}.
In the present notation, this is the same as saying that $\mu_t = 0$ for
all $t$, and that $M_t = M$ for all $t$.
With these assumptions on the stochastic gradient, it is shown that
for an arbitrary convex function, the best achievable rate for a convex
objective function is that $\nmeu{\gJt} = O(t^{-1/2})$.
Now suppose a function $J(\cdot)$ satisfies both Standing Assumptions (S1), (S2)
and the (PL) property.
Thus there exists a constant $K$ such that
\bd
K J(\bth_t) \leq \nmeusq{\gJt} \leq 2L J(\bth_t) .
\ed
Then, as per Corollary \ref{coro:61}, it follows that
$J(\bth_t) = o(t^{-\l})$ and $\nmeusq{\gJt} = o(t^{-\l})$
for every $\l<1$.
There is virtually no difference between $O(t^{-1})$ and $o(t^{-\l})$
for all $\l < 1$.
Thus our results extend the bounds from \cite{Arjevani-et-al-MP23}
from convex functions to a somewhat larger class, namely those that satisfy
Assumptions (S1) and (S2) as well as the Polyak-Lojasiewicz condition.

Next, we study stochastic gradient methods based on function evaluations alone.
The Simultaneous Perturbation SA (SPSA), described in \eqref{eq:1221},
is typical of this approach.
In this equation, two function evaluations are used at each step;
however, there exist approaches that use only one function evaluation
at each step.
For the stochastic gradient of \eqref{eq:1221}, the quantities $\mu_t$
and $M_t$ satisfy
\be\label{eq:3115}
\mu_t = O(c_t), \quad M_t^2 = (1/c_t^2) .
\ee
A more general approach, somewhat reminiscent of the Runge-Kutta method,
is proposed in \cite{Pach-Bhat-Pras-arxiv22}, wherein $k+1$ function
evaluations are used at each step, leading to
\be\label{eq:3116}
\mu_t = O(c_t^k), \quad M_t^2 = (1/c_t^2) ,
\ee
which reduces to the above when $k = 1$.
This observation raises the question as to whether there is an
``optimal'' choice of the ``increment'' $c_t$, so as to achieve the
fastest convergence.
Specifically, suppose we choose $c_t = \Th(t^s)$ for some exponent $s$.
What is the choice of $s$ that maximizes the bound $\nu$ in \eqref{eq:3112}?

\begin{corollary}\label{coro:62}
Suppose all hypotheses of Theorem \ref{thm:62} hold.
Suppose $\mu_t$, $M_t$ satisfy \eqref{eq:3115} for arbitrary
increment $c_t$, and that $c_t = \Th(t^{-1})$.
Then the optimal choice for the exponent $s$ is $1/3$.
Then, with $\al_t = O(t^{-(1-\phi)})$, by choosing 
$\phi = \e > 0$ arbitrarily small, and $s = (1-\e)/3$, we get
\be\label{eq:3117}
J(\bth_t) , \nmeusq{\gJt} = o(t^{-\l}) , \fa \l < 1/3 .
\ee
More generally, suppose $\mu_t$, $M_t$ satisfy \eqref{eq:3116}
for arbitrary increment $c_t$.
Then, with $\al_t = O(t^{-(1-\phi)})$, by choosing 
$\phi = \e > 0$ arbitrarily small, and $s = (1-\e)/(k+2)$,
we get
\be\label{eq:3118}
J(\bth_t) , \nmeusq{\gJt} = o(t^{-\l}) , \fa \l < k/(k+2) .
\ee
\end{corollary}

\begin{proof}
With $c_t = O(t^{-s})$, it is already known from \cite{Kief-Wolf-AOMS52}
that
\bd
\mu_t = O(c_t) = O(t^{-s}) , \quad M_t^2 = O(1/c_t^2) = O(t^{2s}) .
\ed
	Hence we can apply Theorem \ref{thm:62} with $\g = s, \d = 2s$.
Then the rate of convergence becomes $o(t^{-\l})$ whenever $\l \in (0,\nu)$, and
\bd
\nu = \min \{ 1 - 2(\phi + s) , s - \phi \} .
\ed

To motivate the proof, we depict these two inequalities and the
``optimal'' choice of $s$ for the case $k=1$.
Figure \ref{fig:3} depicts the  two inequalities
\bd
1 - 2(\phi+s) \geq 0, s - \phi \geq 0,
\ed
or
\bd
\phi + s \leq 0.5, \phi \leq s .
\ed
The blue line depicts when both parts of the minimum defining $\nu$ are equal,
namely $3s + \phi = 1$.
Along this line, $\mu$ is maximum when $s = 1/3$ and $\phi = 0$, where 
$\mu = 1/3$.
In reality the inequalities should be strict.
Hence, for arbitrarily small $\e > 0$, we can choose
\bd
\phi = \e , \quad s = \frac{1-\e}{3} , \quad
\mu  = \frac{1}{3} - \frac{4\e}{3} .
\ed
But since $\e$ is arbitrary, this works out to $\mu < 1/3$.
Hence \eqref{eq:3117} follows.
In the case of general $k$, we have
\bd
1 - 2(\phi+s) = ks - \phi, \mbox{ or } (k+2)s + \phi = 1.
\ed
So by choosing $\phi = \e$, we get
\bd
s = \frac{1-\e}{k+2} , \quad \mu = \frac{k(1-\e)}{k+2} - \e
= \frac{k}{k+2} - \e \frac{2k+2}{k+2} .
\ed
Again, since $\e$ is arbitrary, \eqref{eq:3118} follows.
This completes the proof.
\end{proof}

It is worth noting that, when $k+1$ function evaluations are carried out,
not only is the convergence rate faster, but the step sizes also become
larger ($O(t^{k/(k+2)})$).

\bfig[htp]
\bc
\btp[line width = 2pt]

\draw [->] (0,0) -- (2.5,0) node [right] {$\phi$} ;
\draw [->] (0,0) -- (0,2.5) node [above] {$s$} ;

\draw (0,2) node [left] {$1/2$} -- (1,1) ;
\draw (0,0) -- (1,1) ;
\draw [blue] (0,4/3) node [left] {$1/3$} -- (1,1) ;
\draw [red,dotted,line width = 1pt] (0,1) node [left] {$1/4$} -- (1,1) ;
\draw [red,dotted,line width = 1pt] (1,0) node [below] {$1/4$} -- (1,1) ;

\etp
\caption{Feasible combinations of $(\phi,s)$}
\label{fig:3}
\ec
\efig

\textbf{Remarks:} Now we discuss the significance of Corollary \ref{coro:62}
and its relationship to previously known results.
\ben
\item The analysis in \cite{Arjevani-et-al-MP23} on the achievable
rates of convergence applies only when the stochastic gradient is unbiased
$(\mu_t = 0$ for all $t$), and its conditional variance is bounded.
When only function evaluations are used to construct a stochastic gradient,
these assumptions do not hold.
Corollary \ref{coro:62} partially fills this gap.
\item In \cite{Nesterov-FCM17}, the authors study what would be called
Simultaneous Perturbation SA with two measurements (but with a Gaussian
perturbation vector instead of Rademacher perturbations).
It is shown that the iterations converge at the rate $J(\bth_t) = O(t^{-1/2})$.
However, there is no error in the measurements, and the objective function
is restricted to be convex.
In contrast, in the present situation, a rate of $o(t^{-\l})$ is achieved
for $\l < 1/3$ even in the presence of measurement errors, and for a class
of nonconvex objective funtions.
Moreover, by choosing $k = 2$ in the approach of \cite{Pach-Bhat-Pras-arxiv22},
that is, by carrying out \textit{three} function evaluations at each step,
the rate goes up to $\l < 1/2$, the same as in \cite{Nesterov-FCM17}.
By letting $\kai$, one can make $\l$ arbitrarily close to one.
In the view of the authors, this last observation is only of
theoretical interest.
\een

\section{Convergence of Stochastic Approximation}\label{sec:SA}

In this section, we state some new theorems on the convergence of the
stochastic approximation (SA) algorithm in \eqref{eq:111}.
These theorems build upon their counterparts in \cite{MV-MCSS23},
and are applicable to more general conditions on the measurement error
$\bxt$.
As with our study of SGD, these assumptions are the most general to date.
Note that the notation used here is slightly different from
that in \cite{MV-MCSS23}.

To refresh the reader's memory, the basic SA algorithm aims to find a
zero of a $\C^1$ function $\f : \R^d \ap \R^d$.
One begins with a (possibly random) initial guess $\bth_0$, after which
the update rule is
\be\label{eq:71}
\bth_{t+1} = \bth_t + \al_t [ \f(\bth_t) + \bxt ] ,
\ee
where $\al_t$ is a nonnegative-valued and possibly random
step size, and $\bxt$ is the measurement error.
We begin with the assumptions on the function $\f(\cdot)$ in \eqref{eq:71}.
\ben
\item[(F1)] $\f(\cdot)$ is globally Lipschitz-continuous with constant $S$.
\item[(F2)] The equation $\f(\bths) = \bz$ has a unique solution, which
is assumed to be $\bths = \bz$, by shifting coordinates if necessary.
\een
Next we state the assumptions on the measurement error $\bxt$.
In analogy with \eqref{eq:1221a}, let us define
\be\label{eq:72}
\z_t = E_t(\bxt) , \quad \bzt = \bxt - \z_t .
\ee
Then it follows that
\be\label{eq:73}
E_t(\bzt) = \bz , CV_t(\bxt) = CV_t( \bzt) ,
E_t( \nmeusq{\bxt} ) = \nmeusq{\z_t} + CV_t(\bzt) .
\ee
With this notation, we state the assumptions on the measurement error $\bxt$.
\ben
\item[(N1)] There exists a sequence of constants $\{ \mu_t \}$ such that
\be\label{eq:74}
\nmeu{E_t(\bxt)} = \nmeu{ \x_t }
\leq \mu_t ( 1 + \nmeu{\bth_t} ) , \fa t \geq 0 ,
\ee
\item[(N2)] There exists a sequence of constants , $\{ M_t \}$ such that
\be\label{eq:75}
CV_t(\bzt) = E_t(\nmeusq{\bzt}) 
\leq M_t^2 ( 1 + \nmeusq{\bth_t} ) , \fa t \geq o .
\ee
\een
Note that \eqref{eq:74} is comparable to \eqref{eq:313}, while \eqref{eq:75}
is comparable to \eqref{eq:314}.
Also, \eqref{eq:74} is more general than \cite[Eq.\ (8)]{MV-MCSS23}, because
in that paper, the error $\bxt$ is assumed to have zero conditional
expectation.
Similarly, \eqref{eq:75} is  more general than \cite[Eq.\ (9)]{MV-MCSS23}
in that the bound on the conditional variance is allowed to vary with time
(and be unbounded).

Our proof makes use of some concepts from Lyapunov stability theory.
The reader is directed to \cite{Mv-Book93} for details on this topic.
Recall that if a function $V: \R^d \ap \R$
is $\C^1$, then the associated function $\Vd : \R^d \ap \R$ is defined as
\bd
\Vd(\bth) = \IP{\nabla V(\bth)}{\f(\bth)} .
\ed
Thus, if $\bth(\cdot)$ is the solution of $\dot{\bth} = \f(\bth)$,
then
\bd
\Vd(\bth(t)) = \frac{d}{dt} V(\bth(t)) .
\ed
With this background, we state our assumptions on the Lyapunov function.
\ben
\item[(L1)] $ \nabla V$ is $\C^1$ and $L$-Lipschitz continuous, and $ \nabla V(\bz) = 0$.
\item[(L2)] There exist positive constants $a, b$ such that
\be\label{eq:76}
a \nmeusq{\bth} \leq V(\bth) \leq b \nmeusq{\bth} , \fa \bth \in \R^d .
\ee
\een

Now we state our results on the convergence of the SA algorithm of
\eqref{eq:71}.
To avoid a lot of repetition, we state a \textbf{standing assumption}:
\ben
\item[(S)] Assumptions (F1), (F2), (N1), (N2), (L1), (L2) hold.
\een

\begin{theorem}\label{thm:71}
Throughout, it is supposed that Assumptions (S) hold.
\ben
\item Suppose that
\be\label{eq:77}
\sum_{t=0}^\infty \al_t^2 < \infty , \quad
\sum_{t=0}^\infty \al_t \mu_t  < \infty , \quad
\sum_{t=0}^\infty \al_t^2 M_t^2 < \infty ,
\ee
and in addition that $\Vd(\bth) \leq 0$ for all $\bth$.
Then $\{ V(\bth_t) \}$ and $\{ \nmeu{\bth_t} \}$ are bounded, and in addition,
$V(\bth_t)$ converges to some random variable as $\tai$.
\item Suppose that, in addition to \eqref{eq:77}, it is also the case that
\be\label{eq:78}
\sum_{t=0}^\infty \al_t = \infty ,
\ee
and in addition, there exists a function $\psi: \R_+ \ap \R_+$ 
belonging to Class $\B$ such that
\be\label{eq:79}
\Vd(\bth) \leq - \psi(\nmeu{\bth}) , \fa \bth \in \R^d .
\ee
Then $V(\bth_t) \ap 0$ and $\bth_t \ap \bz$ as $\tai$.
\item
Suppose that \eqref{eq:77} and \eqref{eq:78} hold, and that there
exists a constant $c > 0$ such that
\be\label{eq:710}
\Vd(\bth) \leq -c \nmeusq{\bth}  , \fa \bth \in \R^d .
\ee 
Further, suppose there exist constants $\g > 0$ and $\d \geq 0$ such that
\bd
\mu_t = O(t^{-\g}), \quad M_t = O(t^\d) ,
\ed                                                            
where we take $\g = 1$ if $\mu_t = 0$ for all sufficiently large $t$,
and $\d = 0$ if $M_t$ is bounded.
Choose the step-size sequence $\{ \al_t \}$ as
$O(t^{-(1-\phi)})$ and $\OM(t^{-(1-C)})$
where $\phi$ is chosen to satisfy
\bd
0 < \phi < \min \{ 0.5 - \d , \g \} ,
\ed
and $C \in (0,\phi]$.
Define
\be\label{eq:711}       
\nu := \min \{ 1 - 2( \phi + \d) , \g - \phi \} .
\ee
Then $\nmeusq{\bth_t} = o(t^{-\l})$
for every $\l \in (0,\nu)$.
In particular, by choosing $\phi$ very small, it follows that
$\nmeusq{\bth_t} = o(t^{-\l})$
\be\label{eq:712}
\l < \min \{ 1 - 2 \d , \g \} .
\ee
\een
\end{theorem}

\textbf{Remark:}
Note that \cite[Theorem 6]{MV-MCSS23} provides sufficient conditions
involving the function $\f(\cdot)$,
for the existence of a Lyapunov function $V$ that satisfies
Assumption (L1), \eqref{eq:76}, and \eqref{eq:710}.

\begin{proof}
Because the proof of Theorem \ref{thm:71}
is essentially a repetition of those of Theorems
\ref{thm:61} and \ref{thm:62}, it is just sketched.

It follows from applying \cite[Eq.\ (2.4)]{Ber-Tsi-SIAM00} 
to \eqref{eq:111} that
\begin{eqnarray*}
V(\bth_{t+1}) & \leq & V(\bth_t) + \al_t \IP{\nabla V(\bth_t)}{\f(\bth_t)}
+ \al_t \IP{\nabla V(\bth_t)}{\bxt} \\
& + & \al_t^2 \frac{L}{2} \nmeusq{\f(\bth_t) + \bxt} .
\end{eqnarray*}
Applying $E_t(\cdot)$ to both sides, using \eqref{eq:72} and \eqref{eq:73},
and applying the definition of $\Vd(\cdot)$, gives
\beq
E_t(V(\bth_{t+1})) & \leq & V(\bth_t) + \al_t \Vd(\bth_t) 
+ \al_t \IP{\nabla V(\bth_t)}{\z_t} \nonumber \\
& + & \al_t^2 \frac{L}{2} [ \nmeusq{\f(\bth_t)} + 2 \IP{\f(\bth_t)}{\z_t}
+ \nmeusq{\z_t} + E_t(\nmeusq{\bzt}) ] . \label{eq:714}
\eeq
Now we observe that
\bd
\nmeu{\f(\bth_t)} \leq S \nmeu{\bth_t} , \quad
\nmeu{ \nabla V(\bth_t)} \leq L \nmeu{\bth_t} ,
\ed
\bd
2 \nmeu{\bth_t} \leq 1 + \nmeusq{\bth_t} \leq
1 + a^{-1} V(\bth_t).
\ed
Substituting these into \eqref{eq:714} gives a bound in the form
\be\label{eq:715}
E_t(V(\bth_{t+1})) \leq ( 1 + f_t ) V(\bth_t) + g_t + \al_t \Vd(\bth_t) ,
\ee
where, as before, $f_t$ (not to be confused with $\f(\bth_t)$) and $g_t$
are sequences consisting of these five terms:
\bd
\al_t^2 , \; \al_t \mu_t , \; \al_t^2 \mu_t , \; \al_t^2 \mu_t^2 ,
	\; \al_t^2 M_t^2 .
\ed
Hence, if \eqref{eq:77} holds, then $\{ f_t \} , \{ g_t \}$
belong to $\ell_1$.

Now we can sketch the remainder of the proof.
\ben
\item
If $\Vd(\bth) \leq 0$ for all $\bth$, then we can replace \eqref{eq:715} by
\bd
E_t(V(\bth_{t+1})) \leq ( 1 + f_t ) V(\bth_t) + g_t .
\ed
Then, from Item 1 of Theorem \ref{thm:51}, it follows that $V(\bth_t)$
is bounded and converges to a random variable.
\item
Suppose that, in addition to \eqref{eq:77},
both \eqref{eq:78} and \eqref{eq:79} also hold.
Then \eqref{eq:715} becomes
\bd
E_t(V(\bth_{t+1})) \leq ( 1 + f_t ) V(\bth_t) + g_t - \al_t \psi(\nmeu{\bth_t}).
\ed
Now Item 2 of Theorem \ref{thm:61} shows that $V(\bth_t) \ap 0$ as $\tai$,
which in turn implies that $\bth_t \ap \bz$ as $\tai$.
\item Suppose that, in addition to \eqref{eq:77},
both \eqref{eq:78} and \eqref{eq:710} also hold.
Then \eqref{eq:715} becomes
\begin{eqnarray*}
E_t(V(\bth_{t+1})) & \leq & ( 1 + f_t ) V(\bth_t) + g_t - \al_t c \nmeusq{\bth_t} \\
& \leq & ( 1 + f_t ) V(\bth_t) + g_t - \al_t c b^{-1} V(\bth_t) .
\end{eqnarray*}
The remainder of the analysis follows as in Theorem \ref{thm:62}.
One can obtain bounds on the rate of convergence of $V(\bth_t)$ to zero,
which in turn translate into bounds on the rate of convergence of
$\nmeusq{\bth_t}$ to zero, using \eqref{eq:76}.
The details are routine and left to the reader.
\een
This completes the proof.
\end{proof}

\begin{corollary}\label{coro:71}
Suppose that $\mu_t = 0$ for all $t$, and that $M_t$ is bounded.
Then, by choosing $\phi = O(t^{-(1-\e)})$ with
$\e>0$  arbitrarily small,
we can ensure that $V(\bth_t) , \nmeusq{\bth_t}$ are $o(t^{-\l})$
for all $\l < 1$.
\end{corollary}

The proof of the corollary is omitted as it is straight-forward.

\section{Conclusions}\label{sec:Conc}

In this paper, we have studied the convergence properties of the Stochastic
Gradient Descent (SGD) method for finding a stationary point
of a given $\C^1$ objective function
$J(\cdot): \R^d \ap \R$.
The objective function is not required to be convex.
Rather, it has to satisfy either a weaker version
of the Kurdyka-Lojasiewicz (KL) condition which we denote as
the (KL') property,
or the Polyak-Lojasiewicz (PL) condition.
Either of these assumptions ensures that $J(\cdot)$ belongs
to the class of ``invex'' functions, which have the
property that every stationary point is also a global minimizer.
When $J(\cdot)$ satisfies the (KL') property,
we have shown that the iterations $J(\bth_t)$ converge to the global
minimum of $J(\cdot)$.
Next, when $J(\cdot)$ satisfies the stronger (PL) property,
we are also able to derive estimates on the rate of
convergence of $J(\bth_t)$ to its limit.
While some results along these lines have been published in the past,
our contributions contain two distinct improvements.
First, the assumptions on the stochastic gradient are more general
than elsewhere.
Specifically, the assumptions are as stated in \eqref{eq:313} and
\eqref{eq:314}.
Second, we establish almost sure convergence, and not just convergence
in expectation.
Since any stochastic algorithm generates a single sample path of a
stochastic process, it is very useful to know that almost all sample paths
converge to the desired limit.
Using these results, we show that for functions satisfying the PL property,
the convergence rate is the same as the best-possible rate for convex
functions.
We have also studied SGD when only function evaluations are permitted.
In this setting, we have determined the ``optimal'' increments, 
that is, the optimal choice of the perturbation sequence.
Using the same set of ideas, we have established the global convergence
of the Stochastic Approximation (SA) algorithm, with two improvements
over existing results.
First, our assumptions on the measurement error are more general
compared to the existing literature.
Second, we also derive bounds on the rate of convergence of the SA algorithm
under appropriate assumptions.

There are several directions for future research
that are worth exploring.
Until now, stochastic gradient methods either update \textit{a single component}
of the argument $\bth_t$ at each iteration, or the entire vector.
One can think of an intermediate approach, wherein at each iteration
\textit{some but not necessarily all} components of $\bth_t$ are updated.
This might be called ``Block'' Asynchronous Gradient Descent (BAGD)
for optimization problems.
An analog for the problem of finding a zero of a function can be called
``Block'' Asynchronous Stochastic Approximation (BASA).
The first problem (BAGD) is studied in \cite{MV-TUKR-arxiv23},
while the second problem (BASA) is studied in a companion paper
\cite{MV-RLK-BASA-arxiv21,MV-RLK-BASA-COT24}.
However, these are just preliminary results, and there is considerable
scope for improvement.

Another promising direction is to apply martingale-based methods to study
``momenum-based'' methods such as Polyak's Heavy Ball method
\cite{Polyak-CMMP64}, or Nesterov's accelerated gradient method
\cite{Nesterov-Dokl83}.
In both \cite{Sebbouh-Gower-Defazio-PMLR21} and \cite{Liu-Yuan-arxiv22},
the Heavy Ball method is analyzed, for convex functions in the former
and strongly convex functions in the latter.
In \cite{Liu-Yuan-arxiv22}, a variant of the
Nesterov Accelerated Gradient (NAG) algorithm is also analyzed.
However, it differs from the ``standard'' NAG, in that the step size
goes to zero while the momentum coefficient remains constant, which is the 
inverse of NAG, as reformulated in \cite{Sutskever-et-al-ICML13}.
It would be worthwhile to study whether the ``standard'' NAG can also
be studied using martingale methods.

One topic that we have not explored is that of Polyak-Ruppert
averaging, as reviewed in \cite{Pol-Jud-SICOPT92}.
In principle, averaging leads to a more ``robust'' implementation of SA,
as shown in \cite{Nemirovski-et-al-SIAMO09}.
When the objective function satisfies the (PL) condition,
the convergence rates established here match the ``optimal''
rates in \cite{Arjevani-et-al-MP23} for convex functions.
Therefore the rates established here are also ``optimal'' for objective
functions satisfying the (PL) condition.
This suggests that ``robustness'' cannot be defined simply in terms of
the rate of convergence, and that an alternate definition is needed.

\section*{Acknowledgement}

The authors thank two anonymous reviewers 
for extremely helpful comments
and additional references, which have greatly enhanced the paper.
The authors also thank Reviewer No.\ 2 for providing Lemma \ref{lemma:31}
and its proof.

\section*{Funding Information}

The research of MV was supported by the Science and Engineering Research Board,
India.

\section*{Data Availability Statement}

This manuscript  has no associated data.


\end{document}